\documentclass[final,12pt]{clear2022}

\usepackage{microtype}
\usepackage{graphicx}
\usepackage{booktabs} %
\usepackage{url}
\usepackage{amssymb}
\usepackage{mathtools}
\usepackage{hyperref}

\usepackage{bbm}

\usepackage{algorithm2e}

\usepackage{xcolor}

\usepackage{enumitem}

\usepackage{placeins}

\usepackage[capitalize]{cleveref}
\usepackage[subtle,mathspacing=normal]{savetrees}

\usepackage{minitoc} %
\renewcommand \thepart{}
\renewcommand \partname{}

\usepackage{afterpage}

\newtheorem*{rep@theorem}{\rep@title}
\newcommand{\newreptheorem}[2]{%
\newenvironment{rep#1}[1]{%
 \def\rep@title{#2 \ref{##1}}%
 \begin{rep@theorem}}%
 {\end{rep@theorem}}}
\makeatother

\newreptheorem{proposition}{Proposition}
\newreptheorem{theorem}{Theorem}

\newcommand{\tleftarrow}{\xleftarrow{t}}
\newcommand{\trightarrow}{\xrightarrow{t}}

\newcommand{\indep}{\rotatebox[origin=c]{90}{$\models$}}

\usepackage{soul}

\definecolor{WowColor}{rgb}{.75,0,.75}
\definecolor{SubtleColor}{rgb}{0,0,.50}

\newcounter{margincounter}

\title[Typing assumptions improve identification in causal discovery]{Typing assumptions improve identification\\in causal discovery}

\clearauthor{%
 \Name{Philippe Brouillard} \Email{philippe.brouillard@umontreal.ca}\\
 \addr Mila, Université de Montréal \\
 \addr ServiceNow Research
 \AND
 \Name{Perouz Taslakian} \Email{p.taslakian@samsung.com}\\
 \addr Samsung AI Center Montreal%
 \AND
 \Name{Alexandre Lacoste} \Email{alexandre.lacoste@servicenow.com}\\
 \addr ServiceNow Research%
 \AND
 \Name{S\'ebastien Lachapelle} \Email{sebastien.lachapelle@umontreal.ca}\\
 \addr Mila, Université de Montréal%
 \AND
 \Name{Alexandre Drouin} \Email{alexandre.drouin@servicenow.com}\\
 \addr ServiceNow Research%
}

\begin{document}

\maketitle

\doparttoc %
\faketableofcontents %

\begin{abstract}
Causal discovery from observational data is a challenging task that can only be solved up to a set of equivalent solutions, called an equivalence class.
Such classes, which are often large in size, encode uncertainties about the orientation of some edges in the causal graph.
In this work, we propose a new set of assumptions that constrain possible causal relationships based on the nature of variables, thus circumscribing the equivalence class.
Namely, we introduce \emph{typed directed acyclic graphs}, in which variable types are used to determine the validity of causal relationships.
We demonstrate, both theoretically and empirically, that the proposed assumptions can result in significant gains in the identification of the causal graph.
We also propose causal discovery algorithms that make use of these assumptions and demonstrate their benefits on simulated and pseudo-real data.
\end{abstract}


\section{Introduction} \label{sec:introduction}

\defcitealias{10002010map}{1KGP, 2010}

Can the temperature of a city alter its altitude~\citep{peters2017elements}?
Can a light bulb change the state of a switch?
Can the brakes of a car be activated by their indicator light~\citep{dehaan2019}?
Chances are, you did not need to think very hard to answer these questions, since you intuitively understand the implausibility of causal relationships between certain \emph{types of entities}.
This form of prior knowledge has been shown to play a key role in causal reasoning~\citep{griffiths2011bayes, schulz2004causal, gopnik2000detecting}.
In fact, in the absence of evidence (e.g., data), humans tend to reason inductively and use domain knowledge to generalize known causal relationships to new, similar, entities~\citep{kemp2010learning}.

Nonetheless, the elucidation of causal relationships often goes beyond human intuition.
The abundance of large-scale scientific endeavors to understand the causes of diseases~\citepalias{10002010map} or natural phenomena~\citep{runge2019inferring} are good examples.
In such cases, computational methods for \textit{causal discovery} may help reveal causal relationships based on patterns of association in data (see \citet{heinze2018causal} for a review).
The most common setting consists of representing causal relationships as a directed acyclic graph where vertices correspond to variables of interest and edges indicate causal relationships.
Additional assumptions, like the \emph{faithfulness} condition, are then made to enable reasoning about graph structures based on conditional independences in the data.
While these enable data-driven causal discovery, the underlying causal graph can only be identified up to its \emph{Markov equivalence class}~\citep{peters2017elements}, which can often be very large~\citep{he2015counting} thus leaving many edges unoriented.

Inspired by how humans use types to reason about causal relationships, this work explores how prior knowledge about the nature of the variables can help reduce the size of such equivalence classes.
Building on the theoretical foundations of causal discovery in directed acyclic graphs, we propose a new theoretical framework for the case where \emph{variables are labeled by a type}.
Such types can be attributed based on prior knowledge, e.g., via a domain expert.
We then make assumptions on how types can interact with each other, which constrains the space of possible graphs and leads to reduced equivalence classes.
We show, both theoretically and empirically, that when such assumptions hold in the data, significant gains in the identification of causal relationships can be made.

\paragraph{Contributions:}
\begin{itemize}[leftmargin=8mm]
    \item We propose a new theoretical framework for causal discovery where possible causal relationships are constrained based on the type of variables (\cref{sec:tdag}).
    \item We prove theoretical results that guarantee the orientation of all inter-type edges and, in certain conditions, the convergence of the equivalence class to a singleton (identification), when the number of vertices tends to infinity and the number of types is fixed (\cref{sec:identification}).
    \item We present simple algorithms to incorporate our type-based assumptions in causal discovery, along with theoretical results that guarantee their consistency (\cref{sec:algorithm}).

    \item We present an empirical study that illustrates the benefits of our proposed algorithms over a baseline that does not consider variable types (\cref{sec:experiments}).
\end{itemize}

\section{Problem formulation} \label{sec:background}

\paragraph{Causal graphical models.} 
In this work, we adopt the framework of causal graphical models (CGM)~\citep{peters2017elements}.
Let $X = (X_1, \dots, X_d)$ be a random vector with distribution $P_X$.
Let $G = (V, E)$ be a directed acyclic graph (DAG) with vertices ${V = \{v_1, \dots, v_d\}}$. 
Each vertex $v_i \in V$ is associated to variable $X_i$ and a directed edge $(v_i, v_j) \in E$ represents a direct causal relationship from $X_i$ to $X_j$. 
We assume that $P_X$ can be factorized according to~$G$, that is, $$p(x_1, \dots, x_d) = \prod_{i=1}^d p(x_i \mid \text{pa}_i^G),$$ where $\text{pa}_i^G$ denotes the parents of $X_i$ in $G$.\footnote{This is a slight abuse of language. Here, we mean the parents of vertex $v_i$ in $G$.}
From this graph, it is possible to estimate quantities of causal nature (e.g., via do-calculus~\citep{pearl1995causal}). %
However, in many situations, the structure of~$G$ is unknown and must be inferred from data. 

\paragraph{Causal discovery.}
The task of causal discovery consists of learning the structure of $G$ based on observations from $P_X$.
Some assumptions are required to make this possible.
By adopting the CGM framework, we assume:
(i) \emph{causal sufficiency}, which states there is no unobserved variable that causes more than one variable in $X$ and
(ii) the \emph{causal Markov property},
which states that $X_i \indep_G X_j \mid Z \implies X_i \indep_{P_X} X_j \mid Z$,
where $Z$ is a set composed of variables in $X$,
$X_i\,\indep_G\,X_j \mid Z$ indicates that $X_i$ and $X_j$ are $d$-separated by $Z$ in $G$,
and $X_i\,\indep_{P_X}\,X_j \mid Z$ indicates that $X_i$ and $X_j$ are independent conditioned on $Z$.
Additionally, we assume (iii) \emph{faithfulness}, 
which states that $X_i\,\indep_{P_X}\,X_j \mid Z \implies X_i\,\indep_G\,X_j \mid Z$.
Hence, conditional independences in the data can be used to learn about the structure of $G$.

\paragraph{Equivalence classes.}
Even with these assumptions, $G$ can only be recovered up to a \emph{Markov equivalence class} (MEC)~\citep{peters2017elements}, which is the set of all the DAGs that encode exactly the same conditional independences as $G$.
The MEC is often characterized graphically using an \emph{essential graph} or \emph{Completed Partially Directed Acyclic Graph} (CPDAG), which corresponds to the union of all Markov equivalent DAGs~\citep{andersson1997characterization}. While two DAGs are Markov equivalent if and only if they have the same skeleton and \textit{v-structures} (also called \textit{immoralities}) ~\citep{verma1990equivalence}, the CPDAG can contain other oriented edges, resulting from constraints such as not creating cycles or additional v-structures.\footnote{For our terminology related to graphs, we refer the reader to the Appendix A of \cite{andersson1997characterization}.} 
In some cases, e.g., for sparse graphs, the size of the MEC can be huge~\citep{he2016formulas, he2015counting}, significantly limiting inference about the direction of edges in $G$.
Hence, it is a problem of key importance to find new realistic assumptions to shrink the equivalence class.

There have been a wealth of approaches to alleviate this problem. For instance, some have made progress by including data collected under intervention~\citep{hauser2012characterization}, making assumptions about the functional form of causal relationships \citep{peters2014causal, shimizu2006linear} or including background knowledge on the direction of edges~\citep{meek1995causal}.
In this work, we propose an alternative approach, based on background knowledge, where types are attributed to variables and the interaction between types is constrained.

\section{Related work} \label{sec:related_work}

The inclusion of background knowledge in causal discovery aims to reduce the size of the solution space by adding or ruling out causal relationships based on expert knowledge.
Several forms of background knowledge have been proposed, which place various levels of burden on the expert.
Below, we outline those most relevant to our work (see \citet{constantinou2021information} for a review).

\paragraph{Hard background knowledge.} This type of background knowledge is ``hard'' in the sense that it \emph{must be respected} in the inferred graph structures.
Previous works have considered: sets of forbidden and known edges~\citep{meek1995causal}, a known ordering of the variables~\citep{cooper1992bayesian}, partial orderings of the variables~\citep{andrews2020completeness,scheines1998tetrad}, and ancestral constraints~\citep{li2018bayesian, chen2016learning}.
Among these, partial orderings (or \emph{tiered background knowledge}) are the most similar to our contribution.
In this setting, it is assumed that an expert partitions the variables into sets called \emph{tiers}, and orders the tiers such that variables in a later tier cannot cause variables in an earlier tier.
In contrast, while we require an expert to partition variables into sets (by type), we do not assume that an ordering is known \textit{a priori} (see \cref{app:examples_ordering_unknown} for examples).

\paragraph{Soft background knowledge.} A setting similar to ours, where the type of each variable dictates its possible causal relationships, is presented by \citet{mansinghka2012structured}. They propose a Bayesian method to use this prior knowledge in causal discovery. Their work shows the benefits of such priors, but does not investigate this space of graphs and their properties w.r.t. to structure identifiability.

\paragraph{Grouping variables.}
\citet{parviainen2017learning} explore a setting similar to ours, where variables representing different ``views'' on the same entity are aggregated into groups. 
The authors address the problem of learning causal relationships between groups of variables, which they represent as \textit{group DAGs}.
Our work is conceptually different.
First, variables of a given type could correspond to different entities that are similar, rather than multiple views on a common entity.
Second, our focus is different: their goal is to recover a group DAG, while ours is to make assumptions that facilitate the identification of the causal graph in the variable space.
Note, however, that their \textit{strong group causality} assumption leads to graphs that are a subset of the consistent t-DAGs that we will present.

Interestingly, several recent works applying causal discovery to real-world problems rely on expert knowledge that is compatible with our proposed framework.
For example, in their work on Alzheimer's disease, \citet{shen2020challenges} claim that ``edges from biomarkers or diagnosis to demographic variables are prohibited'' and that ``edges among demographic variables are prohibited'', 
clearly reasoning about relationships between types of variables.
Similarly, the work of \citet{flores2011incorporating} outlines an application of tiered background knowledge in a medical case study.
Converting this setting to ours simply involves considering each tier as a variable type.
Hence, the typing assumptions that we propose in this work, and the associated theoretical results, constitute a way of incorporating expert knowledge that is applicable in practice.
\begin{figure*}[t!]
    \centering
    \includegraphics[width=0.9\linewidth]{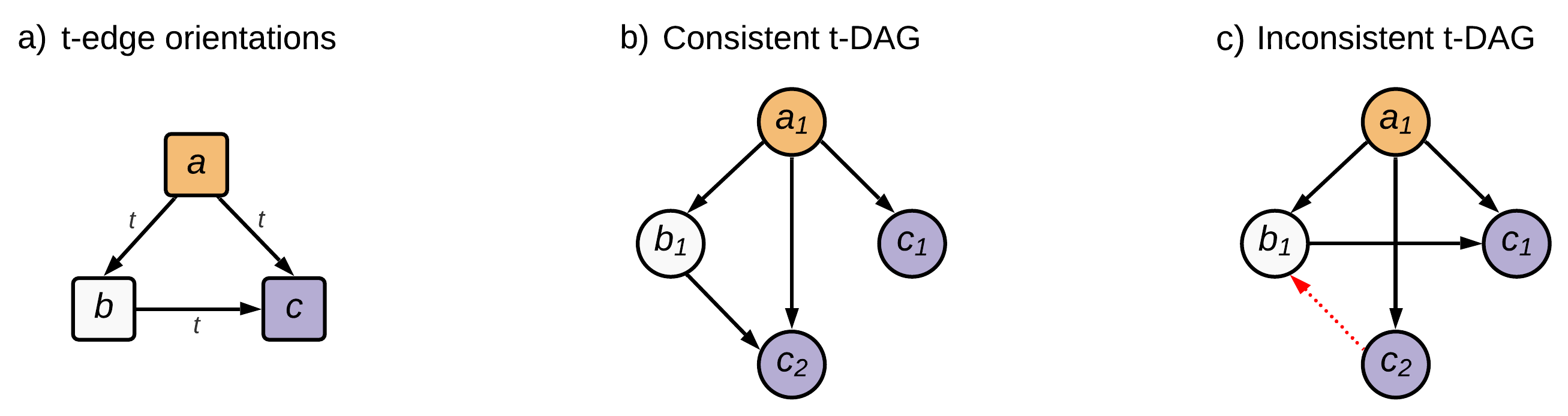}
    \caption{\textbf{(a)} Representation of t-edges orientations, where colors represent the different types $t_a$, $t_b$, and $t_c$. Representation of a t-DAG that is consistent and follows the orientation of the t-edges in (a). \textbf{(c)} Representation of a t-DAG that is not consistent: the red dotted edge $v_{c_2} \rightarrow v_{b_1}$ is not consistent with $v_{b_1} \rightarrow v_{c_1}$ (\cref{def:consistent tdag}).}
    \label{fig:tdag_cartoon}
\end{figure*}
\section{Typed directed acyclic graphs} \label{sec:tdag}

Our work builds on two fundamental structures: \emph{typed} directed acyclic graphs (t-DAG), which are essentially DAGs with typed vertices; and \emph{t-edges}, which are sets of edges relating vertices of distinct types. Formal definitions follow.

\begin{definition}[t-DAG]\label[definition]{def:tdag}
A \emph{t-DAG} $D_T$ with $k$ types is a DAG $D := (V, E)$
augmented with a mapping $T: V \rightarrow \mathcal{T}$ such that the type of $v_i \in V$ is $T(v_i) = t_j \in \mathcal{T}$, where $|\mathcal{T}| = k$.
\end{definition}

\begin{definition}[t-edge]\label[definition]{def:tedge}
A \emph{t-edge} $E(t_i, t_j)$ is the set of edges that goes from a vertex of type~$t_i$ to a vertex of type~$t_j$. More formally,  
$E(t_i, t_j) = \{ (v_k, v_l) \in E \mid T(v_k) = t_i, T(v_l) = t_j \}$ for any pair of types $t_i, t_j \in \mathcal{T}$, s.t., $t_i \not= t_j$.
\end{definition}

For example, the graphs illustrated in \cref{fig:tdag_cartoon}~(b) and (c) are t-DAGs where colors represent types and the set $E(t_a, t_c) = \{(v_{a_1}, v_{c_1}), (v_{a_1}, v_{c_2})\}$ is a t-edge between types $t_a$ and $t_c$.\footnote{To keep the figures simple and readable, throughout the paper we label the vertices of the t-DAGs with the subscripts of the variables (or types) they represent. For example, vertex $a_i$ refers to variable $v_{a_i}$ and vertex $a$ refers to type $t_a$.}

\subsection{Assumptions on type interactions}

We now introduce a new assumption: \emph{type consistency}, which constrains the possible causal relationships that may arise between typed variables.
Put simply, this assumption states that causal relationships between two types of variables can only arise in one common direction.\footnote{See~\cref{app:relaxations} for a discussion of variations and relaxations of this typing assumption.}

\begin{definition}[Consistent t-DAG]\label[definition]{def:consistent tdag}
A consistent t-DAG is a t-DAG where, for every pair of distinct types $t_i, t_j$, if t-edge $E(t_i, t_j) \neq \emptyset$ then we have that $E(t_j, t_i) = \emptyset$. We refer to this structural constraint as type consistency.
For conciseness,  $t_i \xrightarrow{t} t_j$ denotes $E(t_i, t_j) \neq \emptyset$.
\end{definition}

In \cref{fig:tdag_cartoon}~(b), we present an example of a consistent t-DAG. In contrast, the t-DAG shown in \cref{fig:tdag_cartoon}~(c) is not consistent: the t-edge $E(t_c, t_b)$ (purple to white) contains the edge $(v_{c_2}, v_{b_1})$, while the reverse t-edge, $E(t_b, t_c)$, is not empty since it contains $(v_{b_1}, v_{c_1})$.
Notice how the orientation of all {t-edges} (\cref{fig:tdag_cartoon} (a)) fully determines the orientation of edges between variables of distinct types in a consistent t-DAG. %

Note that alternative assumptions could have been considered.
For instance, we could have assumed that t-edges form a DAG (i.e., the types have a partial ordering).
However, the assumptions considered here are less restrictive and, as we demonstrate later, lead to interesting results.

\subsection{Equivalence classes for consistent t-DAGs}
We define the equivalence classes MEC and t-MEC as the set of DAGs and the set of consistent t-DAGs that are Markov equivalent, respectively.

\begin{definition}[MEC] \label{def:mec}
The MEC of a t-DAG $D_T$ is
$M(D_T) := \{ D' \mid D' \sim D_T \}$
where ``$\sim$'' denotes Markov equivalence.
\end{definition}

\begin{definition}[t-MEC] \label{def:t_mec}
The t-MEC of a consistent t-DAG $D_T$ is
$M_T(D_T) := \{ D_T' \mid D_T' \overset{t}{\sim} D_T \}$ where ``$\overset{t}{\sim}$'' denotes Markov equivalence limited to consistent t-DAGs with the same type mapping $T$.
\end{definition}

To represent an equivalence class, we can use an \emph{essential graph}, which corresponds to the union of equivalent DAGs.
The union over graphs is defined as the union of their vertices and edges: $G_1 \cup G_2 := (V_1 \cup V_2, E_1 \cup E_2)$. Also, if $(v_i, v_j), (v_j, v_i) \in E_1 \cup E_2$, then the edge is considered to be undirected.

\begin{definition}[Essential graph]
The essential graph $D^*$ associated to the consistent \mbox{t-DAG} $D_T$ is $$ D^* := \bigcup\limits_{D \in M(D_T)}D.$$
\end{definition}

\begin{definition}[t-Essential graph]\label[definition]{def:tmec}
The t-essential graph $D_T^*$ associated to the consistent {t-DAG} $D_T$ is $$ D_T^* := \bigcup_{D \in M_T(D_T)}D.$$
\end{definition}

\subsection{t-Essential graph properties and size of t-MEC}
We consider some statements that can directly be made about t-essential graphs and the size of t-MEC with respect to their non-typed counterparts. Proofs for the propositions can be found in~\cref{app:proof_proposition}.
First note that for t-DAGs with $k$ types and $d$ vertices, in the limit cases where each variable belongs to a distinct type ($k = d$) or all variables belong to a single type ($k = 1$), type consistency does not impose structural constraints on t-DAGs, i.e., any {t-DAG} is type-consistent and the t-essential graph is identical to the essential graph.

\begin{figure}
    \centering
    \includegraphics[width=0.45\linewidth]{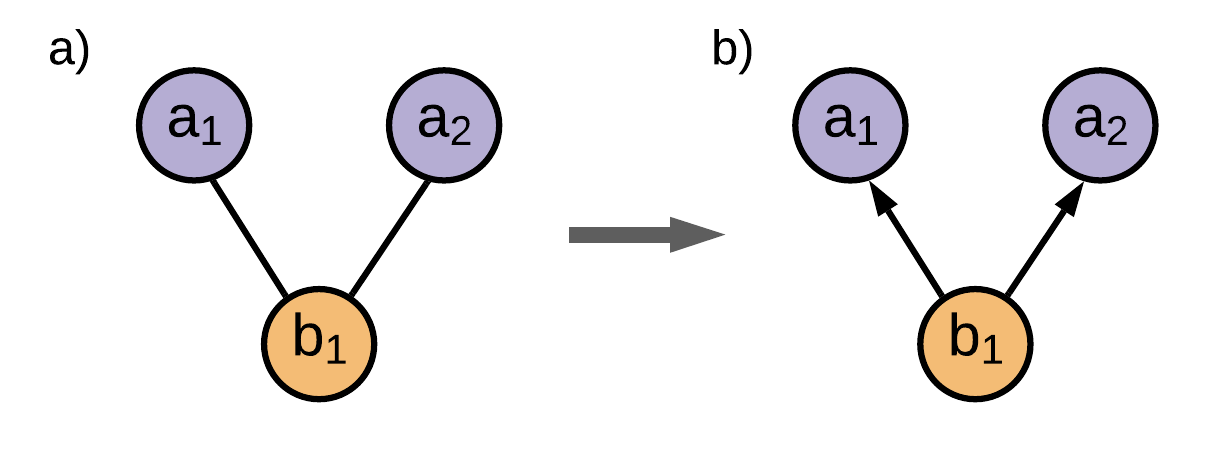}
    \vspace{-3mm}
    \caption{\textbf{(a)} The two-type fork structure. In this illustration, the vertices $v_{a_1}$ and $v_{a_2}$ are of type $t_a$ (purple) and $v_{b_1}$ is of type $t_b$ (orange). 
    \textbf{(b)} Orientation rule: if this structure is encountered in an essential graph, it must be oriented in the t-essential graph.
    }
    \label{fig:two_type_fork}
\end{figure}

However, in general, the t-essential graph is a version of the essential graph with more oriented edges, thanks to the type consistency assumption:

\begin{proposition}
\label[proposition]{prop:essential_subset}
Let $D^*_T$ and $D^*$ be, respectively, the {t-essential} and essential graphs of an arbitrary consistent t-DAG $D_T$.
Then, $D_T \subseteq D^*_T \subseteq D^*$.
\end{proposition}
 
Indeed, type consistency synchronizes the orientation of some edges, resulting in a reduced set of possible orientations.
Some structural properties of the graph may also force the orientation of edges in the t-essential graph.
For instance, akin to v-structures in essential graphs, \textit{two-type forks} (see \cref{fig:two_type_fork}) must be oriented in t-essential graphs.

\begin{proposition}\label[proposition]{prop:mickeymouse}
If a consistent t-DAG $D_T$ contains vertices $v_{a_1}, v_{a_2}, v_{b_1}$ 
with types  $T(v_{a_1}) = T(v_{a_2}) = t_a$, ${T(v_{b_1}) = t_b}$ and $t_a \not= t_b$, 
with edges $v_{a_1} \leftarrow v_{b_1} \rightarrow v_{a_2}$ ($v_{a_1}, v_{a_2}$ not adjacent), then the t-edge $t_b \trightarrow t_a$ is directed in the t-essential graph, i.e., the direction of causation between types~$t_b$ and $t_a$ is known.
\end{proposition}

To see this, note that, under type consistency, there are only two possible orientations: $v_{a_1} \rightarrow v_{b_1} \leftarrow v_{a_2}$ and $v_{a_1} \leftarrow v_{b_1} \rightarrow v_{a_2}$. The first is a v-structure and, thus, must be oriented in the essential graph. If it is not, there is only one possible alternative orientation. Therefore, such edges are always oriented in the t-essential graph.

Furthermore, we can upper bound the size of the t-MEC based on the number of undirected edges in the t-essential graph, as stated in the following proposition.

\begin{proposition}[Upper bound on the size of the t-MEC]
\label[proposition]{prop:random-tmec-ubound}
For any consistent t-DAG~$D_t$, we have $|M_T(D_T)| \leq 2^u \prod_{t_i \in \mathcal{T}} 2^{u_{t_i}}$, where~$u$ and $u_{t_i}$ are respectively the number of undirected t-edges and the number of undirected edges between variables of type $t_i$ (intra-type edges) in the t-essential graph of~$D_t$.
\end{proposition}

From this bound, we can also directly conclude that if the t-essential graph contains no undirected edges, then $|M_T(D_T)| = 1$. In other words, $D_T$ is identified.

\section{Identification for random graphs} \label{sec:identification}

In this section, we explore the benefits of variable typing in causal graph identification through the study of a class of graphs generated at random based on a process inspired by the Erd\H{o}s-R\'enyi random graph model~\citep{erdos1959}.

Assume we are given a set of $k$ types $t_1,\dots,t_k$, 
probabilities $p_1,\dots,p_k \in (0, 1)^k$ of observing each type s.t. $\sum p_i = 1$, 
and a \emph{type interaction matrix} $A \in [0,1]^{k \times k}$ where
each cell $(i, j)$ is the probability $p_{ij}$ that a variable of type $t_i$ is a direct cause of a variable of type $t_j$.
As per \cref{def:consistent tdag} (type consistency), we impose that $\forall i \not= j$, if $p_{ij} > 0$, then $p_{ji} = 0$.

\begin{definition}[Random sequence of growing t-DAG]\label[definition]{def:random_tdags}
We define a random sequence of t-DAGs $(D^n_{T^n})_{n=0}^\infty$ with ${D^n_{T^n} = (V^n, E^n)}$ and $|V^n|=n$, such that $D^0_{T^0} = (\emptyset, \emptyset)$. 
Each new \mbox{t-DAG} $D^{n}_{T^n}$ in the sequence is obtained from $D^{n-1}_{T^{n-1}}$ as follows:
Create a new vertex $v_n$ and sample its type $t$ from a categorical distribution with probabilities $p_1, ..., p_k$.
Let $V^n = V^{n-1} \bigcup \{v_n\}$. Let $T^n(v_n) = t$ and $T^n(v_j) = T^{n-1}(v_j), \forall v_j \in V^{n-1}$;
To obtain $E^n$, for every vertex $v_i \in V^{n-1}$, add the edge $(v_i, v_n)$  to $E^{n-1}$ with probability $p_{T^n(v_i), T^n(v_n)}$. 
\end{definition}

Our main theorem below states that as we add more vertices to such a growing sequence of random t-DAGs, we eventually discover the orientation of all t-edges. We defer the proof to \cref{app:proof_t-edge_convergence}.

\begin{theorem}\label[theorem]{thm:tedges_zero}
Let $(D_{T^n}^n)_{n=0}^\infty$ be a random sequence of growing t-DAGs as defined in \cref{def:random_tdags}, let $U$ be the number of unoriented  t-edges, and let $r_{ij} = -\tfrac{1}{3} \max\big[\ln (1-p_i),  \ln \big(1-p_j p_{ij}(1-p_{jj})\big)\big]$. For $n\geq3$, $p_i, p_j \in (0,1)$, and $p_{ij}, p_{jj} \in [0,1]$, we have: 
\begin{align*}
    P(U > 0) \leq 4 \sum_{i,j \; : \: i \neq j,  p_{ij} > 0} e^{-r_{ij} n}.
\end{align*}
\end{theorem}

To give an intuition of the proof, recall \cref{prop:mickeymouse}, which tells us that any two-type fork structure must be oriented in the t-essential graph, thereby orienting the associated t-edge. 
We thus argue that, as we add more vertices, the probability of observing a two-type fork for arbitrary type pairs converges to $1$.
This argument relies on the fact that the number $k$ of types remains constant throughout as the random t-DAG grows.

From~\cref{thm:tedges_zero} we can derive a result for the case where variables of the same type do not interact ($p_{ii} = 0, \forall i$).
In this case, the t-MEC collapses to a singleton as the graph grows, resulting in %
identification of the true t-DAG.

\begin{corollary}\label[corollary]{cor:identification_pintra_zero}
Let $(D_{T^n}^n)_{n=0}^\infty$ be a random sequence of growing t-DAGs as defined in \cref{def:random_tdags} and let $r_{ij} = -\tfrac{1}{3} \max\big[\ln (1-p_i),  \ln \big(1-p_j p_{ij}\big)\big]$. For $n\geq3$, $p_i, p_j \in (0,1)$, and $p_{ij} \in [0,1]$, the size of the t-MEC converges to $1$ exponentially fast: 
\begin{align*}
    P(|M_T(D^n_{T^n})| > 1) \leq 4 \sum_{i,j \; : \: i \neq j,  p_{ij} > 0} e^{-r_{ij} n}.
\end{align*}
\end{corollary}

\subsection{Empirical validation} \label{sec:empirical_validation}

We conduct an empirical study to validate these theoretical results and further compare the size of the MEC and t-MEC for the case where $p_{ii} > 0$.
As such, we consider t-DAGs of various sizes, randomly generated according to the process described in \cref{def:random_tdags}.
Let $k$ be the number of types in the t-DAG.
We attribute uniform probability to each type, i.e., $p_i=1/k$, $\forall i \in \{1, ..., k\}$.
The type interaction matrix $A$ is defined as follows.
For each pair of types $(t_i, t_j)$, s.t., $i \not= j$, the direction of the t-edge is sampled randomly with uniform probability and we use a fixed probability of interaction $p_{\text{inter}}$.
For example, if the direction $t_i \xrightarrow{t} t_j$ is sampled, then $A_{ij} = p_{\text{inter}}$ and $A_{ji} = 0$.
Furthermore, for each type $t_i$, we attribute a fixed probability $p_{\text{intra}}$ for the occurrence of edges between variables of type $t_i$.

\cref{fig:theory-experiments} shows results for t-DAGs with $n=\{10 \dots 100\}$ vertices, $k=10$ types, $p_{\text{inter}}=0.2$, and various values for $p_{\text{intra}}$ (see \cref{app:random-tdag-experiments} for additional results).
In \cref{fig:theory_tedges_zero}, we clearly see that as, the number of vertices grows, the number of unoriented t-edges tends to zero irrespective of the value of $p_{\text{intra}}$, supporting the statement of \cref{thm:tedges_zero}. 
Furthermore, \cref{fig:theory_identification} clearly shows that for the case where $p_{\text{intra}} = 0$, the size of the t-MEC tends to $1$ (identification) as the number of vertices grows, supporting the statement of \cref{cor:identification_pintra_zero}. 
In sharp contrast, the size of the MEC grows with the number of vertices. Finally, \cref{fig:theory_mec_ratios} shows that the size of the t-MEC can be much smaller than that of the MEC, even when the t-DAG contains intra-type edges ($p_{\text{intra}} > 0$), for which we do not have orientation guarantees.
Interestingly, this also holds when the t-DAGs are dominated by intra-type edges (e.g., $p_{\text{intra}}=0.5$).

It remains an open question to formally quantify the size of the t-MEC vs. the MEC when the graphs contain intra-type edges.
The results in \cref{fig:theory_mec_ratios} suggests that their ratio may be bounded by a quantity that depends on $p_{\text{intra}}$.

\begin{figure}[htbp]
    \floatconts
    {fig:theory-experiments}%
    {
        \caption{\vspace{-6mm} (a) Number of unoriented t-edges w.r.t. the number of vertices (b) Size of the MEC and the t-MEC w.r.t. the number of vertices when $p_{\text{intra}} = 0$ (c) Ratio of the size of t-MEC and the size of MEC w.r.t. the number of vertices.} 
    }
    {%
        \subfigure{%
            \label{fig:theory_tedges_zero}\includegraphics[scale=0.63]{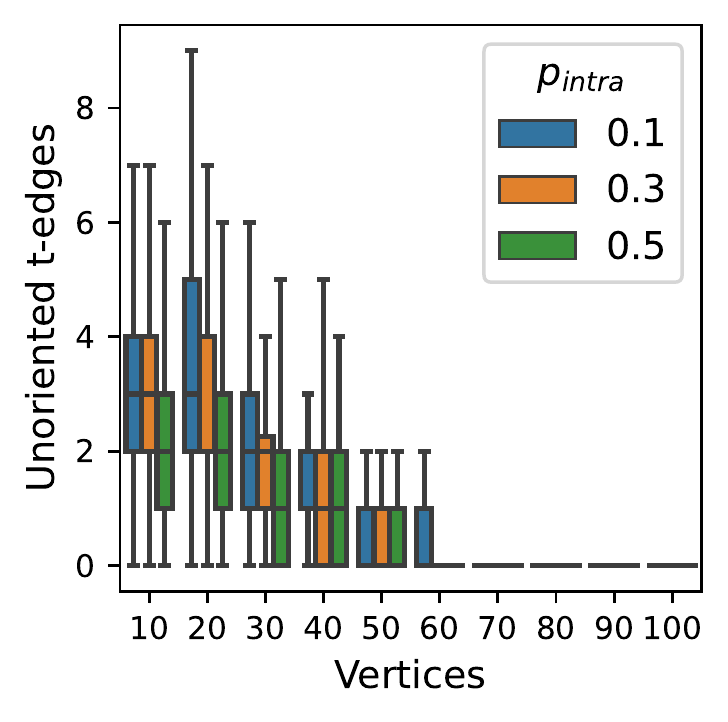}
        }
        \subfigure{%
            \label{fig:theory_identification}\includegraphics[scale=0.63]{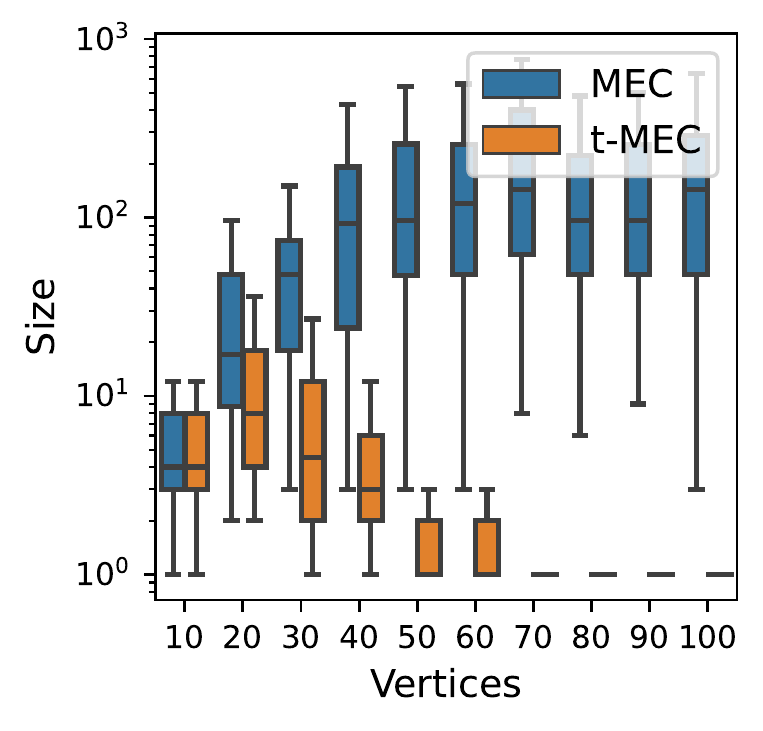}
        }
        \subfigure{%
            \label{fig:theory_mec_ratios}\includegraphics[scale=0.63]{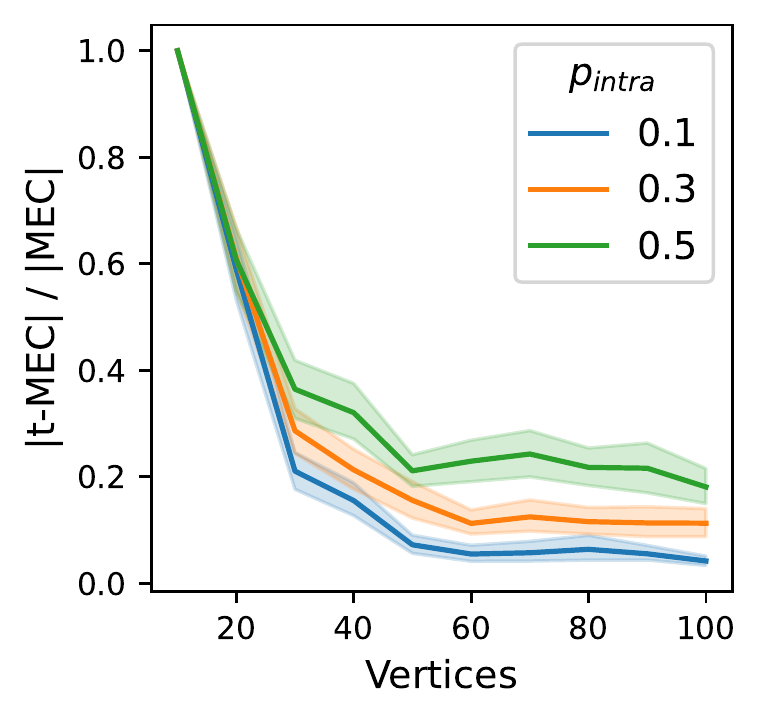}
        }
    }
\end{figure}

\section{Causal discovery algorithms for t-essential graphs} \label{sec:algorithm}

Causal discovery algorithms, such as the PC algorithm~\citep{spirtes2000causation}, are typically consistent w.r.t.\ the MEC.
That is, given infinite samples from the observational distribution entailed by a causal graph, they are guaranteed to recover its essential graph.
Given that the t-MEC is often a small subset of the MEC, it is desirable to find algorithms that can consistently recover t-essential graphs.
In this section, we present such algorithms.

\subsection{From essential to t-essential graph}\label{sec:tmeek}

Given an essential graph, one can recover the t-essential graph by enumerating all Markov equivalent consistent t-DAGs and taking their union. We propose a slightly more efficient approach that propagates t-edge orientations based on the \citet{meek1995causal} orientation rules:
\begin{algorithm}[h!]
\caption{t-Propagation($G, T$)}
\label{alg:tmeek}
    \begin{enumerate}
        \item Enforce type consistency: If there exists an oriented edge between any pair of variables with types ${t_i, t_j \in \mathcal{T}}$ in $G$, assume $t_i \xrightarrow{t} t_j$ and orient all edges between these types.
        \item Apply the \citet{meek1995causal} orientation rules R1-R4 (see their Section 2.1.2) to propagate the edge orientations derived in Step (1).
        \item Repeat from Step (1) until the graph is unchanged.
        \item Enumerate all t-DAGs that can be produced by orienting edges in the resulting graph. Reject any inconsistent t-DAGs and take the union of all remaining t-DAGs to obtain the t-essential graph (see \cref{def:tmec}).
    \end{enumerate}
    \vspace{-7mm}
\end{algorithm}

If $G$ is a graph with the same skeleton, v-structures, and type mapping $T$ as the t-essential graph $D^*_T$, then \cref{alg:tmeek} is guaranteed to recover the corresponding t-essential graph $M_T(D_T)$ (see \cref{app:sec-tmeek}).

Thus, \cref{alg:tmeek} can be used in conjunction with any \mbox{MEC-consistent} causal discovery algorithm to obtain one that is \mbox{t-MEC-consistent}.
However, one major caveat is that, for finite sample sizes, the output of the MEC-consistent algorithm may violate type consistency and cause an irrecuperable failure of t-Propagation (impossibility of orienting t-edges consistently).

Another limitation of this algorithm is its non-polynomial time complexity due to the enumeration in Step (4).
Without it, the algorithm would be \emph{sound}, i.e., it would not orient edges that are unoriented in the t-essential graph, but not \emph{complete}, i.e. some edges that should be oriented in the t-essential graph would remain unoriented.
To see this, consider the \textit{two-type fork} illustrated in \cref{fig:two_type_fork}(a).
The essential graph for this t-DAG would be completely undirected since it contains no v-structures.
This would result in a case where none of the \citet{meek1995causal} rules are applicable and thus, Step (2) would not orient any edges.
The algorithm would therefore stop and return a fully undirected graph.
However, according to \cref{prop:mickeymouse}, the two-type fork should have been oriented.

Nevertheless, even with an additional rule to orient such structures in Step (2) (as illustrated in \cref{fig:two_type_fork}(b)), the algorithm would not be complete without Step (4). In \cref{app:sec-tmeek-counterexamples}, we show more complex counterexamples (non-local and involving multiple t-edges).
It thus 
remains an open question whether it is possible to design a polynomial-time algorithm to find the t-essential graph, as \citet{meek1995causal} and \citet{andersson1997characterization} did for essential graphs. 
Note, however, that the non-polynomial time complexity was found to be non-prohibitive in our experiments.

\subsection{Typed variants of the PC algorithm}\label{sec:tpc}

The previous algorithm suffices to achieve consistency w.r.t. the t-MEC, but it may fail to produce type-consistent outputs when used with finite sample sizes.
Here, we show that it is possible to modify the PC algorithm~\citep{spirtes2000causation} to obtain a t-MEC-consistent algorithm that always produces type-consistent outputs.

Recall how the PC algorithm works; it proceeds in three phases: 1) infer the graph's skeleton using conditional independence tests\footnote{The conditional independence test must be chosen based on the nature of the data. We use  FIT~\citep{chalupka2018fast}, since it is non-parametric and applies to both continuous and discrete data.},
2) orient all v-structures in the skeleton, 3) apply the \citet{meek1995causal} rules to orient as many edges as possible given the edges oriented in Phase (2).
To ensure type-consistent outputs, it suffices to modify Phases (2) and (3) to ensure that, when applicable, we orient whole t-edges instead of single edges.

For Phase (3), the modification is simple: replace the usual procedure by t-Propagation (\cref{alg:tmeek}).
For Phase (2), one must deal with errors that may arise when applying conditional independence tests to samples of finite size.
Indeed, it is possible to observe edges from the same t-edge involved in both v-structures and two-type forks (see \cref{fig:two_type_fork}), suggesting multiple orientations when only one is possible. 
We consider two simple strategies for dealing with such ambiguities that we outline below. These lead to two \emph{Typed-PC (TPC)} algorithms that are both t-MEC-consistent, but differ in their empirical performance, as we show in \cref{sec:experiments}.
Pseudo-codes and proofs of consistency are given in \cref{app:tpc}.

\begin{itemize}
    \item \textbf{TPC-naive}: Use the first encountered structure (v-structure or two-type fork) to orient each t-edge. Naturally, this naive strategy is error-prone.
    \item \textbf{TPC-majority:} We know that only one orientation is possible for t-edges. Hence, look at each structure that would trigger an orientation (v-structures and two-type forks) and choose the orientation based on the most frequent type of structure.
\end{itemize}

\section{Experiments}\label{sec:experiments}

We conduct causal structure learning experiments\footnote{Implementations of the algorithms and code for the experiments are available at \mbox{\url{https://github.com/ElementAI/typed-dag}}.} to compare the performance of our proposed t-MEC-consistent algorithms with that of a baseline which does not make use of variable types.
The t-MEC-consistent algorithms are: TPC-naive and TPC-majority (\cref{sec:tpc}) and PC~\citep{spirtes2000causation} augmented with t-Propagation (\cref{sec:tmeek}).
The baseline is the classical PC algorithm.
The methods are compared in terms of Structural Hamming Distance (SHD) between their output and the ground-truth t-essential graph (see \cref{app:sec-shd-tess} for details).
We base the comparison on synthetic and pseudo-real datasets.

\paragraph{Synthetic data.} We consider graphs randomly generated according to \cref{sec:empirical_validation} with $20$ vertices and $5$ types.\footnote{See \cref{app:additional_exp} for additional results.} The probabilities of connection $p_{\text{inter}}$ and $p_{\text{intra}}$ vary in $\{$(0.2, 0), (0.4, 0), (0.3, 0.1), (0.2, 0.2)$\}$. The $(0.2, 0)$ configuration results in sparse graphs with no intra-type edges. The others configurations lead to denser graphs with similar densities, but which differ in the abundance of intra-type edges. For each type of graph, we explore multiple parametric forms for the causal relationships: linear, nonlinear additive noise model (ANM) \citep{buhlmann2014cam}, and nonlinear non-additive model using neural networks (NN) \citep{kalainathan2018sam}.
Finally, for each type of graph and parametric form, we generate $50$ different consistent t-DAGs and draw $10$k samples from their observational distribution.

\paragraph{Pseudo-real data.} We consider the following Bayesian networks from the \href{https://www.bnlearn.com/bnrepository/}{Bayesian Network Repository} (number of variables in parentheses): \texttt{sachs} (11), \texttt{child} (20), \texttt{insurance} (27), \texttt{alarm} (37), \texttt{hailfinder} (56), \texttt{win95pts} (76).
For each network, the conditional probabilities are fitted to real-world data sets, enabling the generation of pseudo-real data.
To assign types to variables, we randomly partition their topological ordering into groups of expected size $5$ (see \cref{app:sec-pseudoreal}).
We consider $50$ random type assignments and for each, we sample $50$k observations from the Bayesian network.

\begin{figure}
    \centering
    \includegraphics[width=\textwidth]{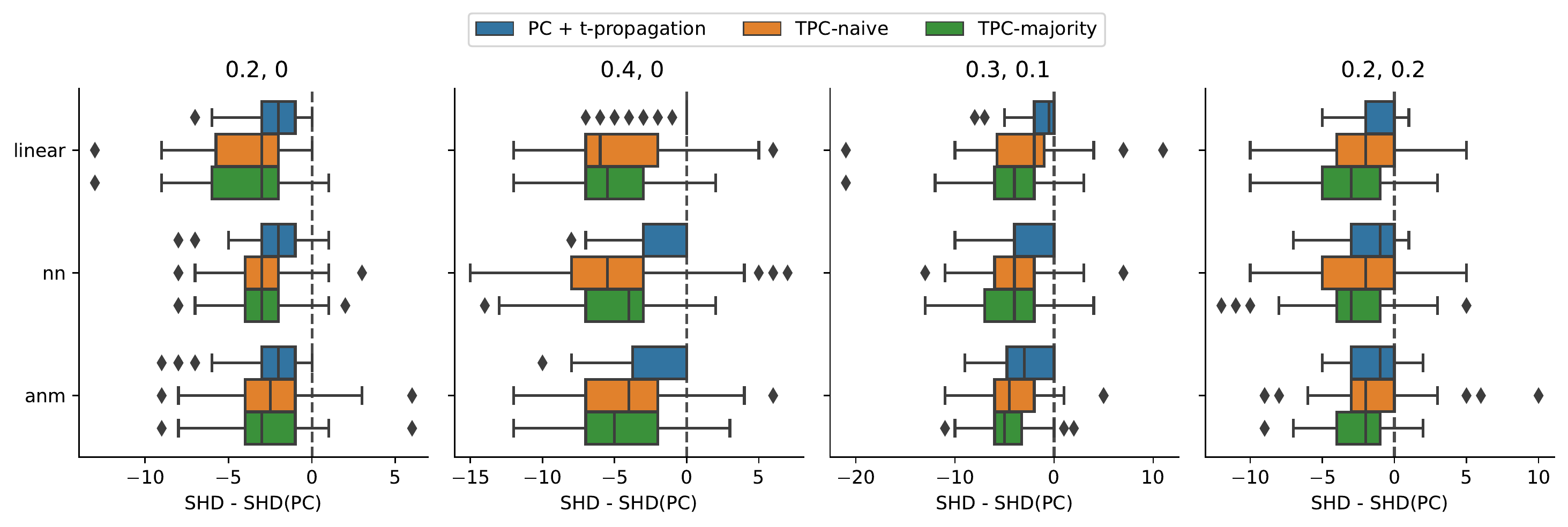}
    \caption{\vspace{-4mm}Results on simulated data shown as improvement in SHD w.r.t. the PC baseline (lower is better). The title of each subplot indicates the ($p_{\text{inter}}, p_{\text{intra}}$) configuration. 
    }
    \label{fig:boxplot_simulated}
\end{figure}

\begin{figure}
    \centering
    \includegraphics[width=\textwidth]{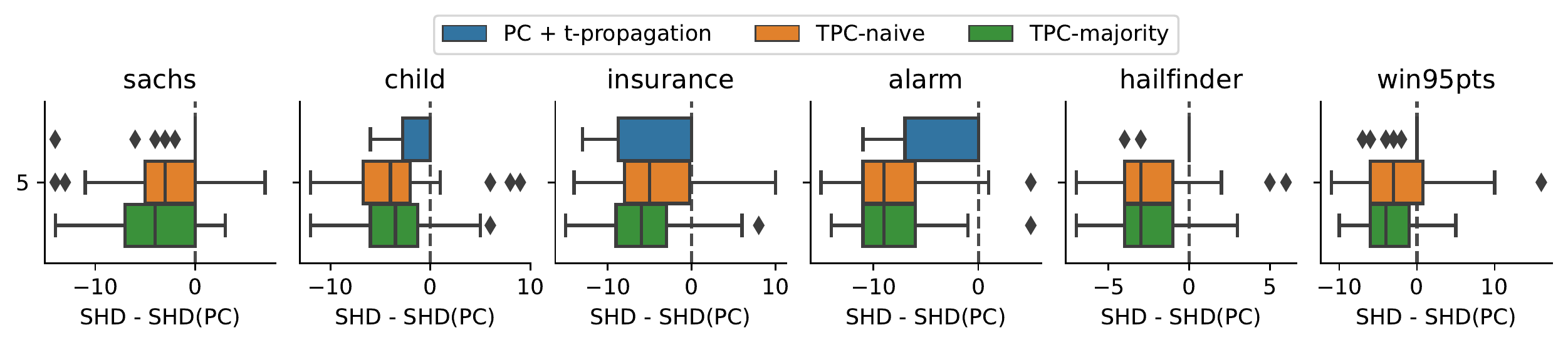}
    \caption{\vspace{-4mm}Results on pseudo-real data shown as improvement in SHD w.r.t. the PC baseline (lower is better). The title of each subplot indicates the underlying network. }
    \label{fig:boxplot_pseudoreal}
\end{figure}

\paragraph{Results.}
The results reported in \cref{fig:boxplot_simulated,,fig:boxplot_pseudoreal} show the improvement in SHD w.r.t. PC.
As expected, all t-MEC-consistent algorithms generally outperform this type-agnostic baseline, achieving lower SHD w.r.t.\ the t-essential graph.
The PC\ +\ t-Propagation method sometimes underperforms compared to the TPC algorithms, mainly because it fails to produce a type-consistent output\footnote{In this case, it simply returns the output of PC.} ($49\%$ of the time).
Finally, as expected, TPC-majority tends to yield better or equal performance compared to TPC-naive, since it is less vulnerable to errors in t-edge orientations that may arise at Phase (2) of the algorithm (see \cref{sec:tpc}).

\section{Discussion} \label{sec:discussion}

In this work, we address an important problem in causal discovery: the fact that it is often impossible to identify the causal graph precisely, due to the size of its Markov equivalence class.
This is particularly true for sparse graphs, where the size of the MEC grows super-exponentially with the number of vertices \citep{he2015counting}.
Our theoretical and empirical results clearly demonstrate that there exist conditions under which our variable-typing assumptions greatly shrink the size of the equivalence class.
Hence, when such assumptions hold in the data, gains in identification are to be expected. We also propose methods to recover the t-essential graph from data and, using several synthetic and pseudo-real data sets, show that these perform better than their type-agnostic counterparts.

We note that the new assumptions that we introduce can be used in conjunction with other strategies to shrink the size of equivalence classes, such as considering interventions~\citep{hauser2012characterization}, hard background knowledge on the presence/absence of edges~\citep{meek1995causal}, or functional-form assumptions~\citep{peters2014causal, shimizu2006linear}.

While this work focuses on causal discovery, it would be interesting to explore the implications of our theoretical framework for causal inference, i.e., the estimation of causal effects.
For instance, the small size of t-MECs in comparison to MECs could improve the accuracy of methods that estimate causal effects based on equivalence classes, such as the IDA variant of \citet{perkovic2017interpreting}.
Also, \cite{anand2022effect} recently proposed a method to estimate causal effects in graphs where clusters of variables have unknown relationships. This may prove particularly useful to estimate causal effects based on t-essential graphs with oriented t-edges, but unoriented intra-type edges, since these could be treated as clusters.

In addition, we believe that this work may stimulate new advances at the intersection of machine learning and causality~\citep{scholkopf2021toward, scholkopf2019causality}.
In fact, machine learning algorithms excel at classification, and thus it may be interesting to explore a setting where the variable types are learned based on some variable features.
Type assignments could be learned \emph{in parallel with} causal discovery using recent methods for differentiable causal discovery~\citep{dcdi, NEURIPS2018_e347c514}.
This may further reduce the burden on human experts in cases where types are hard to assign.
As an example, consider the task of learning causal models of gene regulatory networks.
One could train a model to assign types to genes, based on features of their DNA sequence or their categorization in the gene ontology~\citep{gene2004gene}, in a process where types are assigned such as to help causal discovery.

Another interesting future direction would be to use our typing assumptions to perform causal discovery on multiple graphs at once, i.e., \emph{multi-task causal discovery}.
In fact, assume that we are given data for multiple groups of variables that correspond to disjoint systems (no interactions across groups), but that share similar types.
It would be possible to use type consistency (\cref{def:consistent tdag}) to propagate t-edge orientations across graphs.

In conclusion, our type consistency assumption can result in significant gains in identification that can readily be leveraged by modified versions of common causal discovery algorithms. 
We believe that assumptions based on types are truly important since, in addition to facilitating causal discovery, 
they are likely to be a key component of causal reasoning in intelligent agents.

\section*{Acknowledgements}
The authors are grateful to Assya Trofimov, David Berger, Hector Palacios, Jean-Philippe Reid, Nicolas Chapados, Pau Rodriguez, Pierre-André Noël, and Sébastien Paquet for helpful comments and suggestions. They also thank the anonymous reviewers for their thoughtful questions and comments.
Sébastien Lachapelle is supported by an IVADO Excellence PhD scholarship.
\bibliography{ref}

\clearpage
\appendix

\section{Proof of propositions}\label{app:proof_proposition}

\iffalse
\subsection{Proof of \cref{prop:tdag_space}} \emph{$\mathcal{T}\text{-}\mathcal{DAG}(k, d) \subset \mathcal{DAG}(k, d) $ for $1 < k < d$, where $d$ is the number of vertices and $k$ is the number of types.}
\begin{proof}
By definition, a t-DAG is a DAG, so we know that $\mathcal{T}\text{-}\mathcal{DAG}(k, d) \subseteq \mathcal{DAG}(k, d)$. 

Since $k < d$, by the pigeonhole principle, at least two vertices have the same type.
Let $a_1, a_2$, s.t., $T(a_1) = T(a_2)$ be such vertices.
Further, since $k > 1$, there exists at least one vertex of another type, which we denote by $b_1$, s.t., $T(b_1) \not = T(a_1)$.
Now, notice that there must exist at least one $G \in \mathcal{DAG}(k, d)$ which contains the structure $a_1 \rightarrow b_1 \rightarrow a_2$ as a connected component.

%
For the sake of contradiction, assume that $G$ is consistent t-DAG.
We have that $E(a, b) \not= \emptyset$ and $E(b, a) \not= \emptyset$ and thus, $D_T$ cannot be consistent (\cref{def:consistent tdag}).
Hence, there exists at least one DAG with $k > 1$ types and $d > k$ vertices that is not a consistent t-DAG and thus, $\mathcal{T}\text{-}\mathcal{DAG}(k, d) \subset \mathcal{DAG}(k, d) $.
%
%
%
%
\end{proof}

%
%
%

%
%
%
%
%
\fi

%
%

\begin{repproposition}{prop:essential_subset}
Let $D^*_T$ and $D^*$ be, respectively, the {t-essential} and essential graphs of an arbitrary consistent t-DAG $D_T$.
Then, ${D_T \subseteq D^*_T \subseteq D^*}$.
\end{repproposition}

\begin{proof}
It is clear that $D_T \subseteq D^*_T$ and $D_T \subseteq D^*$ since $D^*_T$ and $D^*$ are obtained by undirecting edges from $D_T$. 
Also, since enforcing type consistency can only orient more edges in $D^*$, we have that $D^*_T \subseteq D^*$.
\end{proof}

\begin{repproposition}{prop:mickeymouse}
If a consistent t-DAG $D_T$ contains vertices $v_{a_1}, v_{a_2}, v_{b_1}$ 
with types  $T(v_{a_1}) = T(v_{a_2}) = t_a$, ${T(v_{b_1}) = t_b}$ and $t_a \not= t_b$, 
with edges $v_{a_1} \leftarrow v_{b_1} \rightarrow v_{a_2}$ ($v_{a_1}, v_{a_2}$ not adjacent), then the t-edge $t_b \trightarrow t_a$ is directed in the t-essential graph, i.e., the direction of causation between types~$t_b$ and $t_a$ is known.
\end{repproposition}

\begin{proof}
To prove the statement we show that among all possible orientations $t_b \trightarrow t_a$, $t_b \tleftarrow t_a$, and $t_b\  \overset{t}{\textbf{\textendash}}\ t_a$ of the t-edge, 
the last two are not valid.

For the sake of contradiction, assume $t_b \tleftarrow t_a$ is directed in the t-essential graph of $D_T$. 
This means that there exists a consistent t-DAG $D_1$, having t-edge $t_b \tleftarrow t_a$, that is Markov equivalent to $D_T$. 
Recall that two graphs are Markov equivalent if and only if they have the same skeleton and the same v-structures \citep{verma1990equivalence}. 
Given that $t_b \tleftarrow t_a$, then $D_1$ has the structure $v_{a_1} \rightarrow v_{b_1} \leftarrow v_{a_2}$, which forms a v-structure. 
But since $D_T$ does not contain this v-structure, this contradicts the fact that $D_1$ is Markov equivalent to $D_T$. 

Now, suppose that $t_b\ \overset{t}{\textbf{\textendash}}\ t_a$ is not directed in the t-essential graph of $D_T$. 
This means that there exist two consistent \mbox{t-DAGs} $D_1$ and $D_2$ that are Markov equivalent to $D_T$, 
having the t-edge orientations $t_b \tleftarrow t_a$ and $t_b \trightarrow t_a$, respectively.
As per the argument in the previous case, the existence of $D_1$ leads to a contradiction.

Therefore, the only possible orientation for t-edge between the types $t_a$ and $t_b$ is $t_b \trightarrow t_a$.
\end{proof}

\begin{repproposition}{prop:random-tmec-ubound}
For any consistent t-DAG $D_T$, we have $|M_T(D_T)| \leq 2^u \prod_{t_i \in \mathcal{T}} 2^{u_{t_i}}$, where~$u$ and $u_{t_i}$ are respectively the number of undirected t-edges and the number of undirected edges between variables of type $t_i$ (intra-type edges) in the t-essential graph of~$D_T$.
\end{repproposition}

\begin{proof}
A t-essential graph is the union of consistent t-DAGs. First, note that for edges between variables of different types, we do not have to consider every edge of a consistent t-DAG independently since, by consistency, we have that all the edges included in a t-edge of $D_T$ will always take the same orientation. 
Thus, if a t-edge is undirected in $D^*_T$, it means that there exists at least one consistent t-DAG in $M_T(D_T)$ for each orientation of the t-edge. 
Since each of the $u$ undirected t-edges can take on two directions, there are $2^u$ possible combinations. 
For edges between variables of the same type $t_i$, we have the same upper bound $2^{u_{t_i}}$ where ${u_{t_i}}$ is the number of undirected edges between variables of type $t_i$.
The total of possible combinations is the product of these bounds: $2^u \prod_{t_i \in \mathcal{T}} 2^{u_{t_i}}$.
Note that this is only an upper bound --- some of these orientations are not part of the equivalence class, since they create either a cycle or new v-structures not present in $D_T$.
\end{proof}

\FloatBarrier
\section{Identification results for random graphs}

\subsection{Proof of Convergence}\label{app:proof_t-edge_convergence}
\begin{reptheorem}{thm:tedges_zero}
Let $(D_{T^n}^n)_{n=0}^\infty$ be a random sequence of growing t-DAGs as defined in \cref{def:random_tdags}, let $U$ be the number of unoriented  t-edges, and let $r_{ij} = -\tfrac{1}{3} \max\big[\ln (1-p_i),  \ln \big(1-p_j p_{ij}(1-p_{jj})\big)\big]$. For $n\geq3$, $p_i, p_j \in (0,1)$, and $p_{ij}, p_{jj} \in [0,1]$, we have: 
\begin{align*}
    P(U > 0) \leq 4  \sum_{i,j \; : \: i \neq j,  p_{ij} > 0} e^{-r_{ij} n}.
\end{align*}
\end{reptheorem}

\begin{proof}
To prove the theorem, we leverage \cref{prop:mickeymouse},  which states that a t-edge from type $t_i$ to type~$t_j$ is oriented when we observe a two-type fork structure. Hence, we upper-bound with an exponential function the probability of not observing such an event as $n$ grows (see \cref{fig:convergence-proof}). Without loss of generality, we assume that a vertex of type $t_i$  causes a vertex of type $t_j$. 

\begin{figure}[htb]
    \centering
    \includegraphics[width=0.3\textwidth]{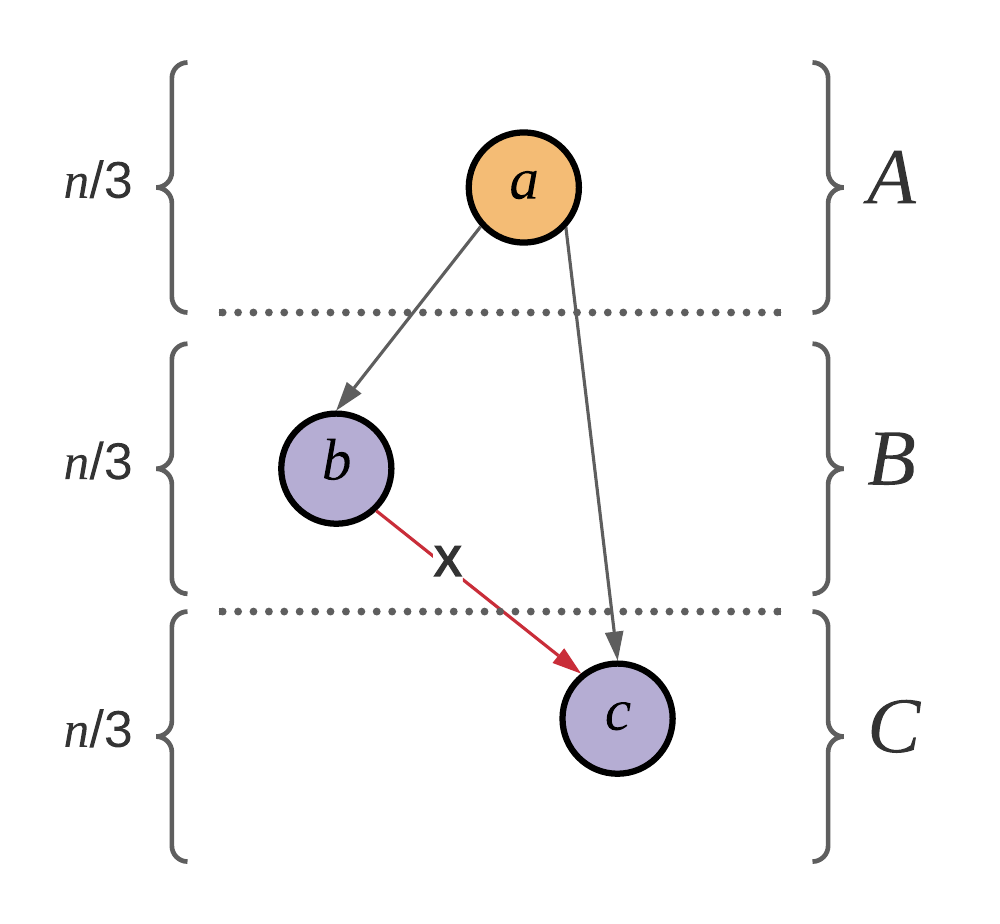}
    \caption{\textbf{Sketch of the proof}. We upper-bound the probability of not observing this two-type fork. The colors correspond to types.}
    \label{fig:convergence-proof}
\end{figure}

\paragraph{Event $\pmb{A}$} Let $A$ be the event of observing at least 1 node of type $t_i$ in the first $m = n/3$ nodes. Using $\alpha = 1-p_i$ as the probability of not sampling a node of type $i$ when we sample a new node, we have:
\begin{align}
    P(A=1 \mid m) = 1 - \alpha^m.
\end{align}
If event $A$ occurs, we define $v_a$ as the first vertex of type $t_i$.

\paragraph{Event $\pmb{B}$} Assuming that event $A$ occurred, we define $B$ as the event of having at least 1 vertex of type $t_j$ caused by $v_a$ during the sampling of the $m=n/3$ nodes that follow. For every new vertex, $p_j \cdot p_{ij}$ represents the probability of this new node being of type $t_j$ and connecting to vertex $v_a$. Using $\beta = 1 -p_j \cdot p_{ij}$ as the inverse of such probability, we have:
\begin{align}
    P(B=1 \mid A=1, m) = 1 - \beta^m.
\end{align}
If $B$ occurs, we define $v_b$ as the first vertex of type $t_j$ connecting to $v_a$. 

\paragraph{Event $\pmb{C}$} Finally, assuming event $A$ and $B$ occurred, we define event $C$ as the event of having at least 1 vertex of type $t_j$ caused by $v_a$ and not connecting to $v_b$, during the sampling of the last $m=n/3$ vertices. For every new vertex, $p_j \cdot p_{ij} \cdot (1-p_{jj})$ represents the probability of this new vertex satisfying these conditions. Using $\gamma = 1- p_j \cdot p_{ij} \cdot (1-p_{jj})$ as the inverse of this probability, we have:
\begin{align}
    P(C=1 \mid A=1, B=1, m) = 1 - \gamma^m.
\end{align}

Let $F_{ij}$ be the event of orienting the t-edge $E(t_i, t_j)$. We thus have:
\begin{align*}
p(F_{ij} \mid n) & \geq P(A=1,B=1,C=1 \mid n) \\
&= P(A=1 \mid m)P(B=1 \mid A=1, m)P(C=1 \mid A=1, B=1, m)\\
&= (1 - \alpha^m) (1 - \beta^m) (1 - \gamma^m) \\
&= 1 - \alpha^m - \beta^m - \gamma^m - \alpha^m \beta^m \gamma^m \color{gray}{ + \alpha^m \beta^m + \alpha^m \gamma^m + \beta^m \gamma^m }\\
&\geq 1 - \alpha^m - \beta^m - \gamma^m - \alpha^m \beta^m \gamma^m 
\end{align*}
The first inequality arises from the fact that many events could lead to t-edges orientation, but we focus only on a subset of them as a sufficient condition. In the last inequality, we drop positive terms to simplify the proof.

To show the convergence rate, we upper-bound the inverse of the probability of event $F_{ij}$ with an exponential function. Since $n\geq 3$, $m\geq1$, $\beta \leq \gamma$, and $\alpha, \beta, \gamma \in (0,1)$, we have:
\begin{align}
    1 - p(F_{ij} \mid n) &\leq   \alpha^m + \beta^m + \gamma^m + \alpha^m \beta^m \gamma^m \nonumber \\
    &= e^{m \ln \alpha} + e^{m \ln \beta} + e^{m \ln \gamma} + e^{m \ln (\alpha \beta \gamma)} \nonumber \\
    &\leq 4 e^{m \ln [ \max(  \alpha,  \beta,  \gamma,  \alpha \beta \gamma) ]} \nonumber \\
    &= 4 e^{m \ln [ \max(  \alpha,   \gamma) ]} \\
    &= 4 e^{-n \cdot r_{ij}} \label{eqn:upper-bound}
\end{align}
Using the union bound for all t-edges, we conclude the proof. We note that the bound is vacuous for small $n$, but converges to zero exponentially fast.
\end{proof}

\FloatBarrier
\subsection{Additional empirical results}\label{app:random-tdag-experiments}

\begin{figure}[htbp]
    \floatconts
    {fig:example2}%
    {
        \caption{Size of the equivalence classes (MEC and t-MEC) w.r.t.various graph properties. a) Number of vertices: the t-MEC shrinks with the number of vertices, while the MEC does not. b) Number of types: The t-MEC grows with the number of types since fewer constraints are placed on the structure of the graphs. The size of the MEC is \emph{mostly} constant. c) Edge density (i.e., probability of connectivity $p$):  The MEC and t-MEC both shrink as connectivity increases.}
    }%
    {%
        \subfigure{%
            \label{fig:exp_bynodes}\includegraphics[width=0.3\linewidth]{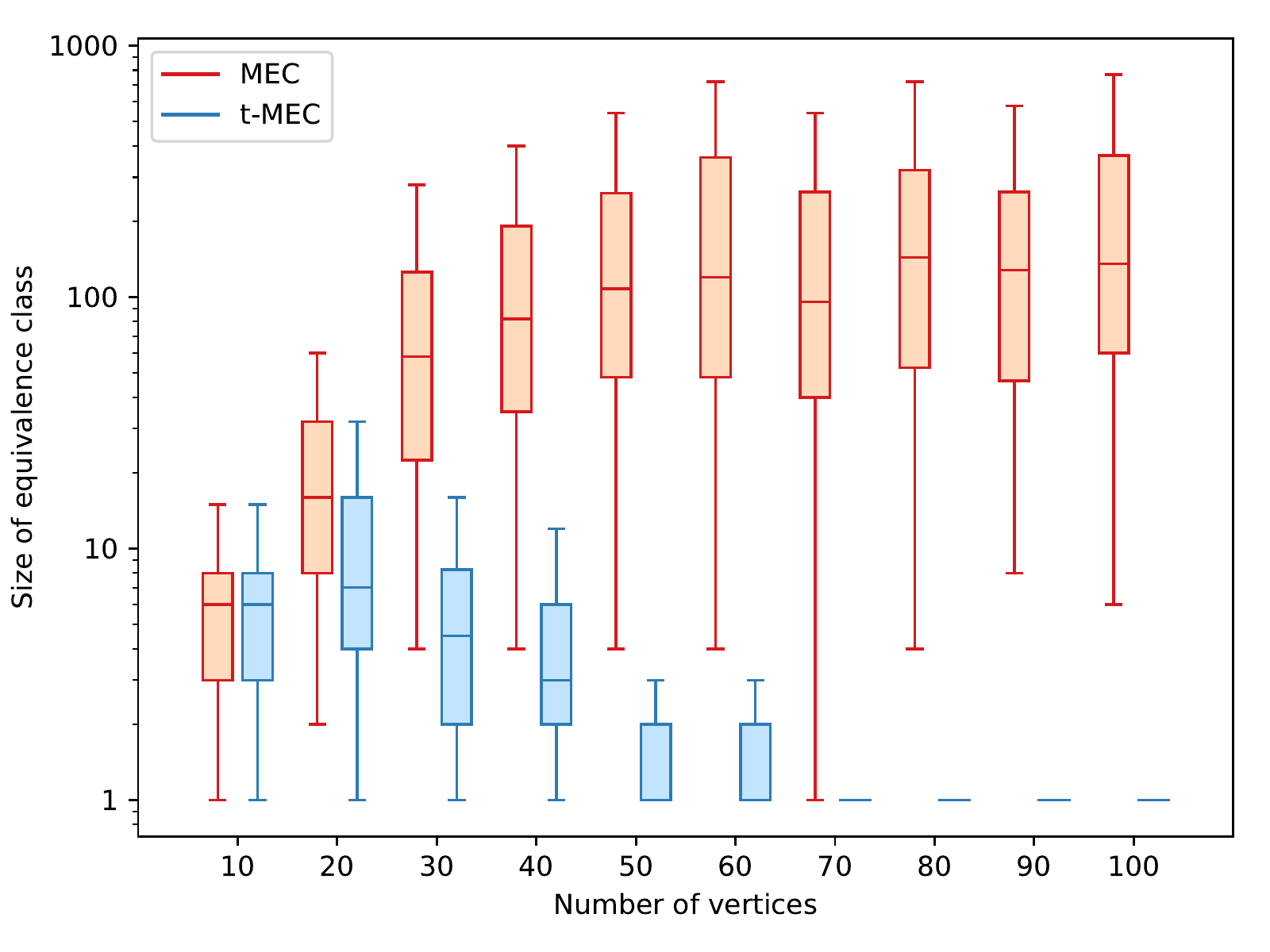}
        }
        \subfigure{%
            \label{fig:exp_bytypes}\includegraphics[width=0.3\linewidth]{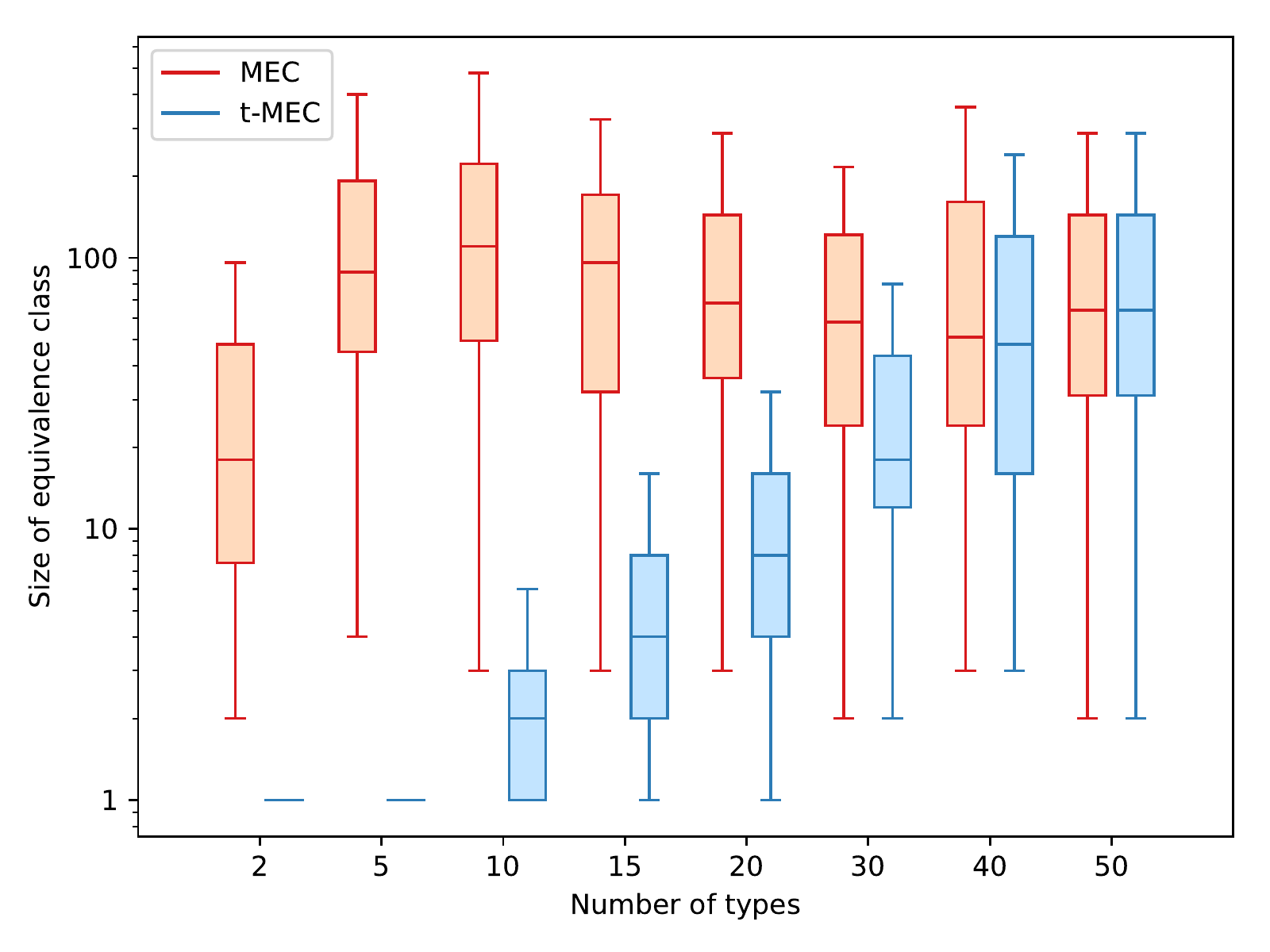}
        }
        \subfigure{%
            \label{fig:exp_bydensity}\includegraphics[width=0.3\columnwidth]{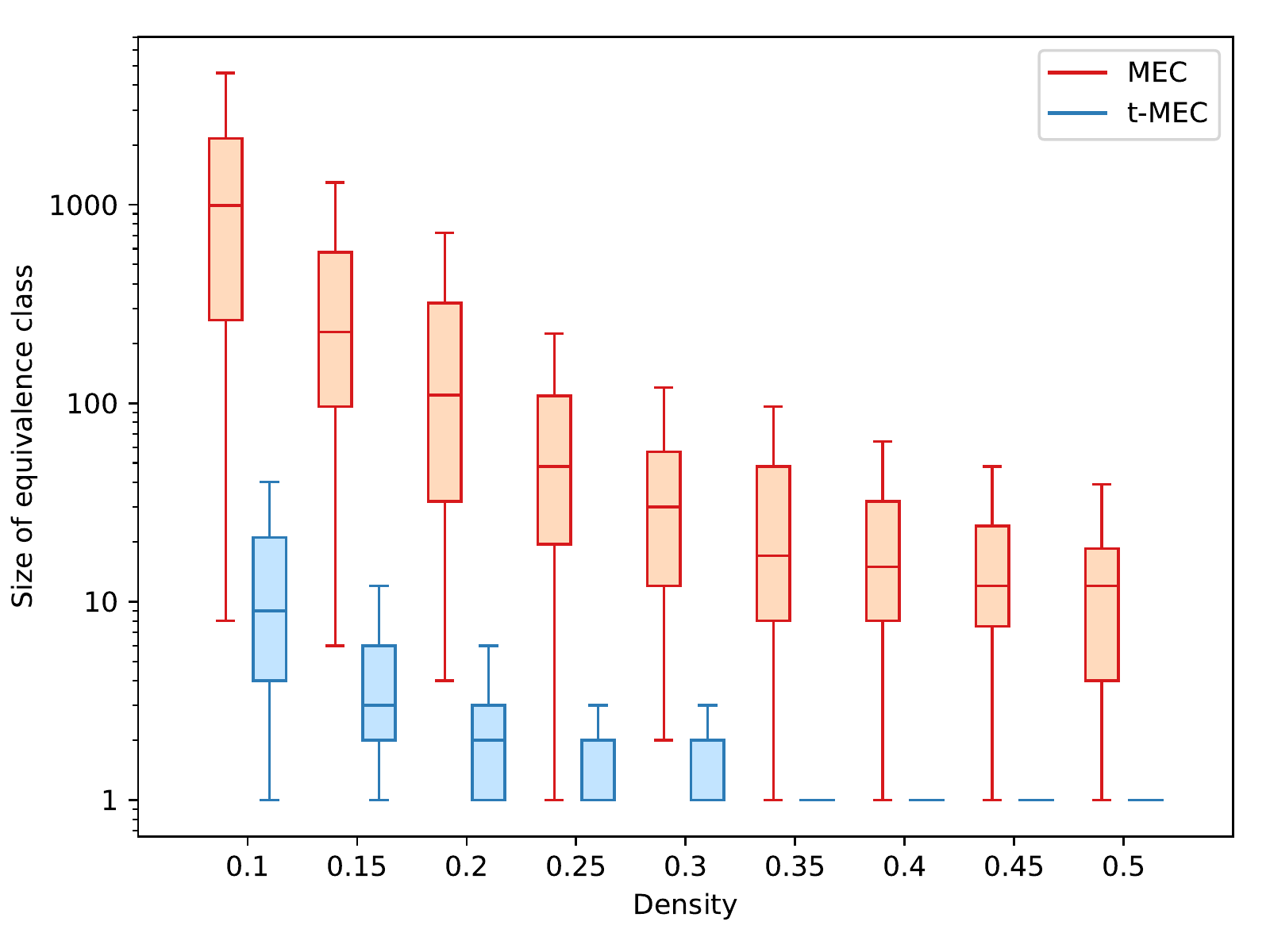}
        }
    }
\end{figure}

For the case where $p_{\text{intra}} = 0$, we performed additional experiments to understand how the size of MECs and t-MECs compare w.r.t. different parameters.
The t-DAGs that we consider are randomly generated according to the process described at \cref{def:random_tdags}. Unless otherwise specified , the number of vertices is $50$, the number of types is $10$, and the probability $p_{\text{inter}}$ is $0.2$.

In \cref{fig:exp_bynodes,,fig:exp_bytypes,,fig:exp_bydensity}, the size of the equivalence classes are compared with respect to the number of vertices, the number of types, and the density, respectively. All boxplots are calculated over 100 random consistent t-DAGs. 

First, in \cref{fig:exp_bynodes} we see that as the number of vertices increases (and the number of types remains constant), the size of the t-MEC converges to $1$, as demonstrated in \cref{sec:identification}. 
In contrast, the size of the MEC first increases and then remains near-constant. 
Notice how the size of the MEC and the t-MEC are identical when the number of vertices equals the number of types; this is because type consistency does not constrain the graph structure.
Second, in  \cref{fig:exp_bytypes}, as the number of types increases, the size of the t-MEC increases. 
This is expected because as the number of types approaches the number of vertices, type consistency imposes fewer structural constraints.
Further, notice that the size of the MEC changes with the number of types, even though it is agnostic to type consistency.
This is because t-DAGs with fewer types (e.g., 2) are more likely to contain v-structures, leading to smaller MECs.
Third, in \cref{fig:exp_bydensity}, as the density increases, the size of the MEC and the t-MEC both decrease. 
This is in line with the observations of \citet{he2015counting}.

In summary, when $p_{\text{intra}}= 0$, all our experiments indicate that the size of the t-MEC is smaller or equal to that of the MEC for random t-DAGs. 
The difference is particularly striking when the number of types is small and the number of vertices is large.
Of particular interest are the results shown in \cref{fig:exp_bynodes}, as they provide empirical evidence for the correctness of \cref{cor:identification_pintra_zero}.

\FloatBarrier
\section{Causal discovery algorithms}\label{app:methods}

\subsection{t-Propagation}\label{app:sec-tmeek}

\subsubsection{Consistency}
\begin{theorem}\label{thm:tmeek}
Given any partially oriented DAG $G$ and t-DAG $D_T$ that has the type mapping $T$ (see \cref{def:tdag}) with the same skeleton and v-structures, and such that $D_T \subseteq G$, t-Propagation($G, T$) is guaranteed to return the t-essential graph of $D_T$.
\end{theorem}

\begin{proof}
First, note that according to \cite{verma1990equivalence}, since $G$ and $D_T$ have the same skeleton and v-structures, they are Markov equivalent. Thus, we are sure that Step 4 is by definition (see \cref{def:tmec}) sound and complete when applied to $G$. In order to prove that the overall algorithm is sound and complete, we just need to show that the previous steps are sound. In other words, edges oriented by the three first steps should have the same orientation in $D_T$. 

If the enforcement of type consistency (Step 1) were not sound, then it would imply that some edge is not type consistent in $D_T$, which is a contradiction of the t-DAG consistency. 

For the Meek rule (Step 2), we show that every orienting rule (R1, R2, R3 and R4) is consistent with $D_T$. By contradiction, we suppose that the orientation $v_a \rightarrow v_b$ given by the rule is wrong and thus $v_a \leftarrow v_b \in D_T$. We show that it leads to v-structures not present in $G$ or cycles.
\begin{itemize}
    \item \textbf{R1.} The pattern $v_c \rightarrow v_a - v_b$, leads to $v_a \rightarrow v_b$. If $v_a \leftarrow v_b \in D_T$ then it would form the new v-structure $v_c \rightarrow v_a \leftarrow v_b$.
    \item \textbf{R2.} The pattern $v_a \rightarrow v_c \rightarrow v_b$ and $v_a - v_b$, leads to $v_a \rightarrow v_b$. If $v_a \leftarrow v_b \in D_T$ then it would imply a cycle.
    \item \textbf{R3.} The pattern $v_a - v_c \rightarrow v_b$, $v_a - v_d \rightarrow v_b$ and $v_a - v_b$, leads to $v_a \rightarrow v_b$. If $v_a \leftarrow v_b \in D_T$ then it would form the new v-structure $v_c \rightarrow v_a \leftarrow v_d$ resulting by applying R2 twice in order to avoid creating a cycle.
    \item \textbf{R4.} The pattern $v_a - v_d \rightarrow v_c \rightarrow v_b$, $v_a - v_b$ and any type of edge between $v_a$ and $v_c$, leads to $v_a \rightarrow v_b$. If $v_a \leftarrow v_b \in D_T$ then it would form the new v-structure $v_b \rightarrow v_a \leftarrow v_d$ resulting by applying R2 twice in order to avoid creating a cycle.
\end{itemize}

The step 3 is also sound since it is simply the repetition of step 1 and step 2. Since the three first steps are sound and the last step is sound and complete, the algorithm is guaranteed to output the t-essential graph.
\end{proof}

\subsubsection{Additional counterexamples for completeness without enumeration}\label[appendix]{app:sec-tmeek-counterexamples}
In this section, we give two additional examples where the t-propagation algorithm presented in \cref{sec:tmeek} would not orient some edges that are oriented in the t-essential graph if the step 4 (enumeration and union) was removed. 

Our first example is interesting because it shows that 
in order to decide the orientation of a t-edge
sometimes several t-edges (possibly not local) have to be considered simultaneously. 
The second counterexample shows that looking only for the direct parent or child of a t-edge is not always sufficient. 

\begin{figure*}[t!]
    \centering
    \includegraphics[width=0.85\linewidth]{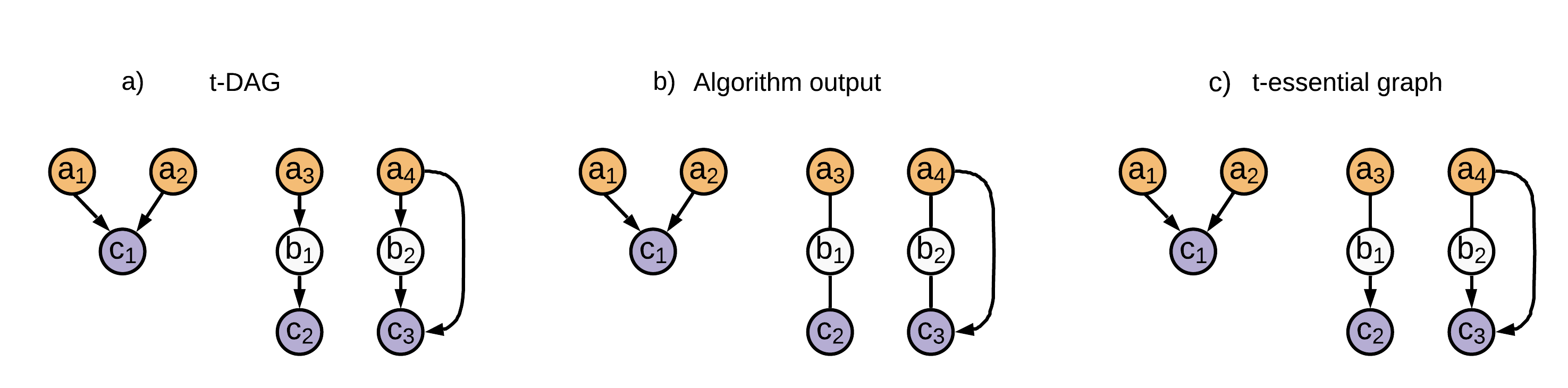}
    \caption{An example where the algorithm would not orient some edges oriented in the t-essential graph. In this case, the orientation of the t-edge is forced by the fact that the reverse orientation would either create a cycle or a new v-structure. \textbf{(a)} The original t-DAG, \textbf{(b)} the algorithm output (which is supposed to be equal to the t-essential graph), \textbf{(c)} the ground-truth t-essential graph}
    \label{fig:counterexample1}
\end{figure*}

\begin{figure*}[t!]
    \centering
    \includegraphics[width=0.85\linewidth]{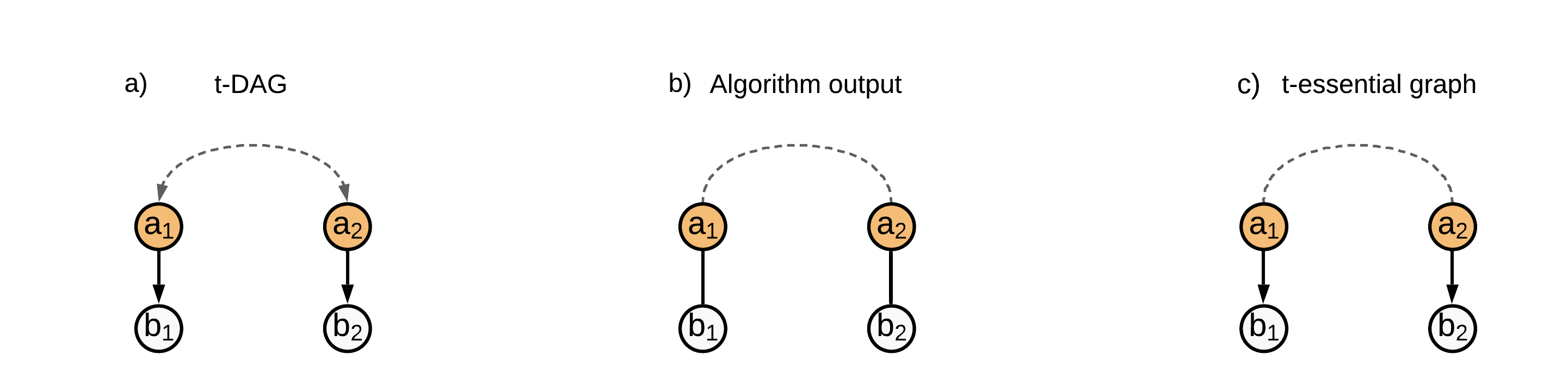}
    \caption{A second example where the algorithm would not orient some edges oriented in the t-essential graph. In this case, the orientation of the t-edge is forced by the fact that the reverse orientation would create a new v-structure. \textbf{(a)} The original t-DAG, \textbf{(b)} the algorithm output (which is supposed to be equal to the t-essential graph), \textbf{(c)} the ground-truth t-essential graph}
    \label{fig:counterexample2}
\end{figure*}

The first example is presented in \cref{fig:counterexample1}. 
Note that vertices denoted by the same letter have the same type. 
The algorithm orients the t-edge $t_a \xrightarrow{t} t_c$ since one of its edges in the t-DAG is part of an v-structure. All other edges are unoriented because they are not covered by any rules. 
However, in the t-essential graph (see \cref{fig:counterexample1}~c) the t-edge ${t_b \xrightarrow{t} t_c}$ is oriented. 
To see why this is the case, consider the four possible orientations of the t-edges $t_a\  \overset{t}{\textbf{\textendash}}\ t_b$  and $t_b\  \overset{t}{\textbf{\textendash}}\ t_c$ (recall that an orientation cannot create a cycle or a new v-structure): 

\begin{enumerate}
    \item $t_a \xrightarrow{t} t_b$, $t_b \xrightarrow{t} t_c$ ~~\indent\indent possible.
    \item $t_a \xrightarrow{t} t_b$, $t_b \xleftarrow{t} t_c$ ~~\indent\indent impossible (creates an v-structure that is not present in the original t-DAG).
    \item $t_a \xleftarrow{t} t_b$, $t_b \xrightarrow{t} t_c$ ~~\indent\indent possible.
    \item $t_a \xleftarrow{t} t_b$, $t_b \xleftarrow{t} t_c$ ~~\indent\indent impossible (creates a cycle).
\end{enumerate}

In the two configurations that are possible, the t-edge $t_b\  \overset{t}{\textbf{\textendash}}\ t_c$ is always oriented as $t_b \xrightarrow{t} t_c$. Thus, this is an essential edge that should have been recovered by the algorithm.

The second example is presented in \cref{fig:counterexample2}. The dashed line between $v_{a_1}$ and $v_{a_2}$ represents a path that does not contain oriented edges in the t-essential graph. Thus, the t-DAG does not contain any v-structure. Without loss of generality, let us consider the dashed line as a chain ${v_{a_1} \leftarrow v_{c_1} \leftarrow v_{c_2} \leftarrow v_{a_2}}$. The algorithm does not orient any t-edges because they are not covered by any rule. However, in the t-essential graph (see \cref{fig:counterexample2}~c) the t-edge ${t_a \xrightarrow{t} t_b}$ is oriented. Consider the impossible orientation ${t_a \xleftarrow{t} t_b}$. Recall that the t-DAG contains no v-structure. Thus, let us orient the edges of the chain as ${v_{a_1} \leftarrow v_{c_1} \leftarrow v_{c_2} \leftarrow v_{a_2}}$ or ${v_{a_1} \rightarrow v_{c_1} \rightarrow v_{c_2} \rightarrow v_{a_2}}$. In both cases, a new v-structure is created (respectively, ${v_{b_1} \rightarrow v_{a_1} \leftarrow v_{c_1}}$ and ${v_{c_2} \rightarrow v_{a_2} \leftarrow v_{b_2}}$) leading to a contradiction. Thus, the t-edge has to be oriented as $t_a \xrightarrow{t} t_b$.

\subsection{Typed-PC} \label{app:tpc}

This section gives additional information on the Typed-PC (TPC) algorithms introduced at \cref{sec:tpc}.
For Phase (1), both algorithms use the same procedure as PC~\citep{spirtes2000causation} to learn the skeleton.
The algorithms differ in their behavior at Phase (2).
The pseudo-codes for the Phase (2) of \emph{TPC-naive} and \emph{TPC-majority} are given at \cref{alg:tpc-naive} and \cref{alg:tpc-majority}, respectively.
For Phase (3), both algorithms use t-Propagation (\cref{alg:tmeek}).
The consistency of these algorithms w.r.t.\ the t-MEC is demonstrated at \cref{thm:tpc_naive} (\emph{TPC-naive}) and \cref{thm:tpc_majority} (\emph{TPC-majority}).

The pseudocodes of the algorithms make use of the following auxiliary functions:
\begin{itemize}
    \item OrientEdge: A method that orients an edge between a pair of variables $v_i, v_j$.
    \item OrientTEdge: A method that orients a t-edge between two types $t_i, t_j$.
    \item Orient: A general orientation method that orients an edge between a pair of variables $v_i, v_j$ and, if $v_i$ and $v_j$ are of distinct types, orients all edges of the corresponding t-edge. This is used when $v_i, v_j$ are not necessarily of the same type.
    \item DisconnectedForks: A method that finds all triplets of vertices $(v_i, v_j, v_k)$ such that $v_i - v_k - v_j$ is in a graph's skeleton, but $v_i - v_j$ is not.
    \item IsOriented: A method that returns $True$ if there is an oriented edge between $v_i$ and $v_j$ and $False$ otherwise.
\end{itemize}

\begin{algorithm2e}
\caption{Orient Forks (Naive)}
\label{alg:tpc-naive}
\SetKwInOut{Input}{input}  
\Input{
    $G=(V, E)$: skeleton of t-DAG with type mapping $T: V \rightarrow \{t_1, ..., t_k\}$\\
    $S_{ij}$: separating sets s.t., $v_i\,\indep_G\,v_j \mid S_{ij}$ for $v_i,v_j \in V$ and $S_{ij} \subset V$
}
\KwOut{$G'$, a more oriented version of $G$}
\For{$(v_i, v_j, v_k) \in \textup{DisconnectedForks}(G)$}{
    \uIf(\tcp*[f]{v-structure}){$k \not\in S_{ij}$}{
        Orient($v_i \rightarrow v_k$)\;
        Orient($v_j \rightarrow v_k$)\;
    }
    \ElseIf(\tcp*[f]{Two-type fork}){$T(v_i) = T(v_j) \not = T(v_k)$}{
        OrientTEdge($T(v_k) \xrightarrow{t} T(v_i)$)\;
    }
}
\end{algorithm2e}

\begin{algorithm2e}
    \caption{Orient Forks (Majority)}
    \label{alg:tpc-majority}
    \SetKwInOut{Input}{input}  
    \Input{
        $G=(V, E)$: skeleton of t-DAG with type mapping $T: V \rightarrow \{t_1, \dots, t_k\}$\\
        $S_{ij}$: separating sets s.t., $v_i\,\indep_G\,v_j \mid S_{ij}$ for $v_i,v_j \in V$ and $S_{ij} \subset V$
    }
    \KwOut{$G'$, a more oriented version of $G$}
    
    \vspace{2mm}\textbf{$\blacktriangleright$ Step 1: orient all multi-type v-structures and two-type forks}\\

    $changed \leftarrow$ True\;
    \While{changed}{
        \tcp{Reinitialize evidence and conditional orientations}
        $E_{ij} \leftarrow 0, \forall i, j \in k \times k$ \tcp*{t-edge orientation evidence counter}
        $C_{ij} \leftarrow \emptyset, \forall i, j \in k \times k$ \tcp*{Conditional orientations}
        
        \vspace{2mm}
        \tcp{Search for t-edge orientation evidence}
        \For{$(v_i, v_j, v_k) \in \textup{DisconnectedForks}(G)$}{
            \uIf(\tcp*[f]{V-structure}){$k \not\in S_{ij} \And \neg(\textup{IsOriented}(v_i, v_k) \And \textup{IsOriented}(v_j, v_k))$}{
                \tcp{Increment t-edge orientation evidence}
                \lIf{$T(v_i) \not= T(v_k)$}{$E[T(v_i), T(v_k)] \mathrel{+}= 1$}
                \lIf{$T(v_j) \not= T(v_k)$}{$E[T(v_j), T(v_k)] \mathrel{+}= 1$}
                
                \tcp{Single-type edges to orient if a t-edge is oriented}
                \lIf{$T(v_i) = T(v_k)$}{
                    $C_{T(v_j), T(v_k)} = C_{T(v_j), T(v_k)} \cup \{ (v_i, v_k) \}$
                }
                \lIf{$T(v_j) = T(v_k)$}{
                    $C_{T(v_i), T(v_k)} = C_{T(v_i), T(v_k)} \cup \{ (v_j, v_k) \}$
                }
            }
            \ElseIf(\tcp*[f]{Two-type fork}){$T(v_i) = T(v_j) \not = T(v_k)$}{
                \tcp{Increment t-edge orientation evidence}
                $E[T(v_k), T(v_i)] \mathrel{+}= 2$\;
            }
        }
        \lIf{$\max E = 0$}{
            $changed \leftarrow False$
        }
        \Else{
            \tcp{Orient t-edge $t_i \xrightarrow{t} t_j$ with max evidence and all edges in $C_{t_i, t_j}$}
            $(t_i, t_j) \leftarrow \text{argmax}\,E$\;
            orientTEdge($t_i \xrightarrow{t} t_j$)\;
            \lFor{$(v_k, v_l) \in C_{t_i, t_j}$}{
                orientEdge($v_k \rightarrow v_l$)
            }
        }
    }
    
    \vspace{2mm}\textbf{$\blacktriangleright$ Step 2: orient all single-type v-structures}\\
    \For{$(v_i, v_j, v_k) \in \textup{DisconnectedForks}(G)$}{
        \If(\tcp*[f]{V-structure}){$T(v_i) = T(v_j) = T(v_k) \And v_k \not\in S_{ij}$}{
            OrientEdge($v_i \rightarrow v_k$)\;
            OrientEdge($v_j \rightarrow v_k$)\;
        }
    }
\end{algorithm2e}
\FloatBarrier

\begin{theorem}[Consistency of TPC-naive to the t-MEC] \label[theorem]{thm:tpc_naive}
Assuming that the causal Markov property holds, under the faithfulness and causal sufficiency assumptions (see \cref{sec:background}), given a suitable conditional independence test, and in the limit of infinite samples from the distribution entailed by a t-DAG $D_T$, the TPC-naive algorithm is guaranteed to recover the t-essential graph $D^*_T$.
\end{theorem}

\begin{proof} 

\textit{\textbf{Phase 1.}} Since we are in the population case, the conditional independence tests will perfectly recover conditional independences in the distribution entailed by the underlying t-DAG $D_T$.
Since we assume faithfulness, this translates into the perfect recovery of $D_T$'s skeleton.

\textit{\textbf{Phase 2.}} Let the input $G$ of \cref{alg:tpc-naive} be such a perfect skeleton. Also, let the input $S_{ij}$ be the set of variables that render $X_i$ and $X_j$ conditionally independent in the distribution. Since \cref{alg:tpc-naive} iterates over all disconnected forks in the skeleton and has access to all ground-truth $S_{ij}$, all v-structures and two-type forks will be correctly identified. Their edges will be oriented correctly in $G'$ via the \texttt{Orient} or \texttt{OrientTEdge} functions. If such edges are part of t-edges, the whole t-edge will be oriented in $G'$.
The orientation of such t-edges is guaranteed to be correct, otherwise, $D_T$ would violate type consistency.

\textit{\textbf{Phase 3.}} We, therefore, have a partially oriented DAG $G'$, that has the same type mapping, v-structures, and skeleton as $D_T$. 
Furthermore, we know that $D_T \subseteq G'$, since the input at Phase (2) is the true skeleton of $D_T$ (fully undirected) and we have shown that all oriented edges in $G'$ will be correct w.r.t.\ $D_T$.
Hence, from \cref{thm:tmeek}, we know that applying t-Propagation to $G'$ is guaranteed to return $D^*_T$, our desired output.
\end{proof}

\begin{theorem}[Consistency of TPC-majority to the t-MEC]\label{thm:tpc_majority}
Assuming that the causal Markov property holds, under the faithfulness and causal sufficiency assumptions (see \cref{sec:background}), given a suitable conditional independence test, and in the limit of infinite samples from the distribution entailed by a t-DAG $D_T$, the TPC-majority algorithm is guaranteed to recover the t-essential graph $D^*_T$.
\end{theorem}

\begin{proof}

\textit{\textbf{Phase 1.}} Since we are in the population case, the conditional independence tests will perfectly recover conditional independences in the distribution entailed by the underlying t-DAG $D_T$.
Since we assume faithfulness, this translates into the perfect recovery of $D_T$'s skeleton.

\textit{\textbf{Phase 2.}}
\begin{itemize}
\item The input $G$ to \cref{alg:tpc-majority} will be the exact skeleton of $D_T$ and the $S_{ij}$ will be the exact set of variables that render $X_i$ and $X_j$ conditionally independent in the distribution.
\item \textit{Step 1.} \cref{alg:tpc-majority} starts by iterating through all multi-type disconnected forks. Since the $S_{ij}$ and the skeleton are perfect, all v-structures and two-type forks will be correctly identified. Notice that, it is impossible that both $E[T(v_i), T(v_k)]$ and $E[T(v_k), T(v_i)]$ are greater than 0 since otherwise, $D_T$ would violate type consistency.
Hence, all t-edges involved in a v-structure or a two-type fork will be oriented correctly.
Moreover, \cref{alg:tpc-majority} keeps track of any single-type edge that is involved in a v-structure with a multi-type edge (via $C_{t_i, t_j}$) and orients it correctly.
Assigning an alternative orientation to such edges would contradict the fact that there is a v-structure in the graph, which we know with certainty.

\item \textit{Step 2.} The last step of \cref{alg:tpc-majority} iterates through all single-type disconnected forks. Since the $S_{ij}$ and the skeleton are perfect, all v-structures will be identified correctly and their edges will be oriented accordingly.
\end{itemize}

\textit{\textbf{Phase 3.}} We, therefore, have a partially oriented DAG $G'$, that has the same type mapping, v-structures, and skeleton as $D_T$. 
Furthermore, we know that $D_T \subseteq G'$, since the input at Phase (2) is the true skeleton of $D_T$ (fully undirected) and we have shown that all oriented edges in $G'$ will be correct w.r.t.\ $D_T$.
Hence, \cref{thm:tmeek}, we know that applying t-Propagation to $G'$ is guaranteed to return $D^*_T$, our desired output.
\end{proof}

\FloatBarrier
\section{Experiments}\label{app:experiments}

\subsection{Data sets} \label{app:datasets}
\subsubsection{Synthetic data sets} \label{app:synthetic_data}
Consistent t-DAGs are generated according to the generative process described in \cref{sec:empirical_validation}. From these graphs, data are generated according to the following SCM:
$$ X_j := f_j(X_{\text{pa}_j^{G}}, N_j) , \forall j$$
For each variable $j$, the noise variables $N_j$ are mutually independent and sampled from $\mathcal{N}(0, \sigma_j^2)\ \ \forall j$, where $\sigma_j^2 \sim \mathcal{U}[0.01, 0.02]$. If the variable is a source, then $\sigma_j^2 \sim \mathcal{U}[1, 2]$. The function $f_j$ and the sampling of their parameters is described in details for each type of function:

\begin{itemize}
    \item The \textit{linear} data sets are generated following  $X_j := w_j^T X_{\text{pa}_j^G} + N_j$  where  $w_j$ is a vector of $|\pi_j^{G}|$ coefficients each sampled uniformly from $[-1, -0.25]\cup[0.25,1]$ (to make sure there are no $w$ close to 0).
    
    \item The \textit{additive noise model} (ANM) data sets are generated following  $X_j := f_j(X_{\text{pa}_j^G}) + N_j$  where the functions $f_j$ are fully connected neural networks with one hidden layer of $10$ units and \textit{leaky ReLU} with a negative slope of $0.25$ as nonlinearities. The weights of each neural network are randomly initialized from $\mathcal{N}(0, 1)$.
    
    \item The \textit{nonlinear with non-additive noise} (NN) data sets are generated following  $X_j := f_j(X_{\pi_j^G}, N_j)$ , where the functions $f_j$ are fully connected neural networks with one hidden layer of $20$ units and \textit{tanh} as nonlinearities. The weights of each neural network are randomly initialized from $\mathcal{N}(0, 1)$. 
\end{itemize}

\subsubsection{Pseudo-real data sets}\label[appendix]{app:sec-pseudoreal}
To generate the data, we used the \href{https://pypi.org/project/bnlearn/}{bnlearn} Python package (Version 0.4.3) and conditional probability tables taken from the \href{https://www.bnlearn.com/bnrepository/}{Bayesian Network Repository} (\texttt{.bif} files). For all Bayesian networks, the graph structure and the conditional probability come from real-world data sets. 
Inspired by the method used by \citet{constantinou2021information} to simulate tiered background knowledge, we assigned types to variables by randomly partitioning the topological ordering of the DAGs into groups. The size of the groups can vary, but their expected size is fixed. Note that such types do not necessarily indicate similarity in the nature of the variables, but they do capture the fact that one variable precedes the other. In \cref{fig:app-insurance-dataset} we show an example of type assignments (expected group size: $5$) for the \texttt{insurance} Bayesian network, where each color represents a type.

\begin{figure}
    \centering
    \includegraphics[width=\textwidth]{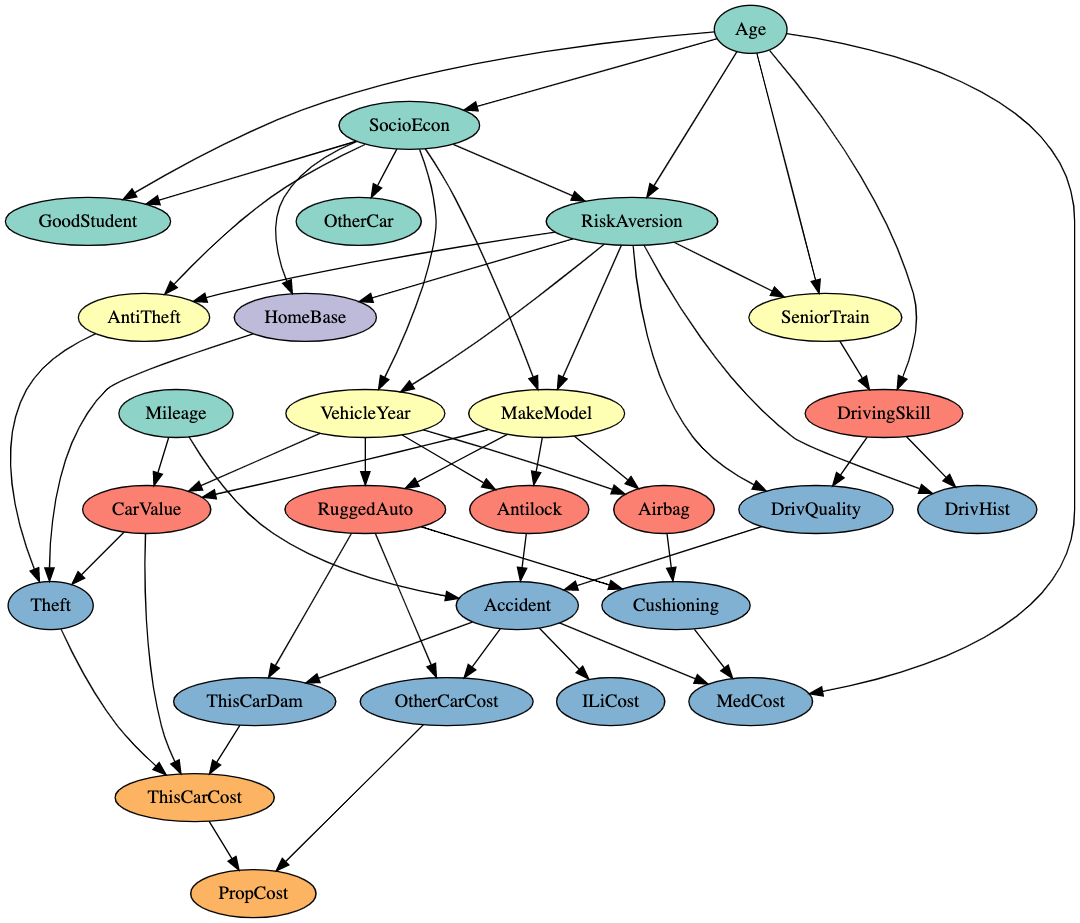}
    \caption{An example t-DAG for the \texttt{Insurance} Bayesian network with an expected \emph{number of variables per type} of $5$. The colors indicate the types.}
    \label{fig:app-insurance-dataset}
\end{figure}

\subsection{SHD with respect to the ground truth t-essential graph}\label{app:sec-shd-tess}
The Structural Hamming Distance (SHD) \citep{tsamardinos2006max} is the number of incorrect edges between two graphs (either superfluous, missing, or reversed edges). In our case, we compare the output of the algorithms to what is identifiable from the data: the ground-truth t-essential graph.

\subsection{Additional experiments} \label{app:additional_exp}

We report additional experiments with the same settings as the main text, but where t-DAGs with a different number of vertices and types are considered. We show, in~\cref{fig:app-simulated-types3,,fig:app-pseudoreal-3types}, results on simulated and pseudo-real data with 3 types (instead of 5). In~\cref{fig:app-simulated-50nodes} we report the result on the simulated data with 50 vertices (instead of 20). For each method, $10$ repeats were performed.

Overall, the main conclusions remain unchanged. Except for the data set $(0.2, 0)$ where PC + t-propagation seems to perform particularly well, the results on data sets with 3 types are similar to the 5-types data sets. For the simulated data with 50 vertices, the difference with the PC baseline seems accentuated, in line with our theoretical results.

Note that the boxplot whiskers represent $\text{Q1} - 1.5 \cdot \text{IQR}$ and $\text{Q3} + 1.5 \cdot \text{IQR} $, where Q1 and Q3 are the first and third quartiles and IQR is the interquartile range. Full results are available in the code repository: \url{https://github.com/ElementAI/typed-dag}.

\begin{figure}
    \centering
    \includegraphics[width=\textwidth]{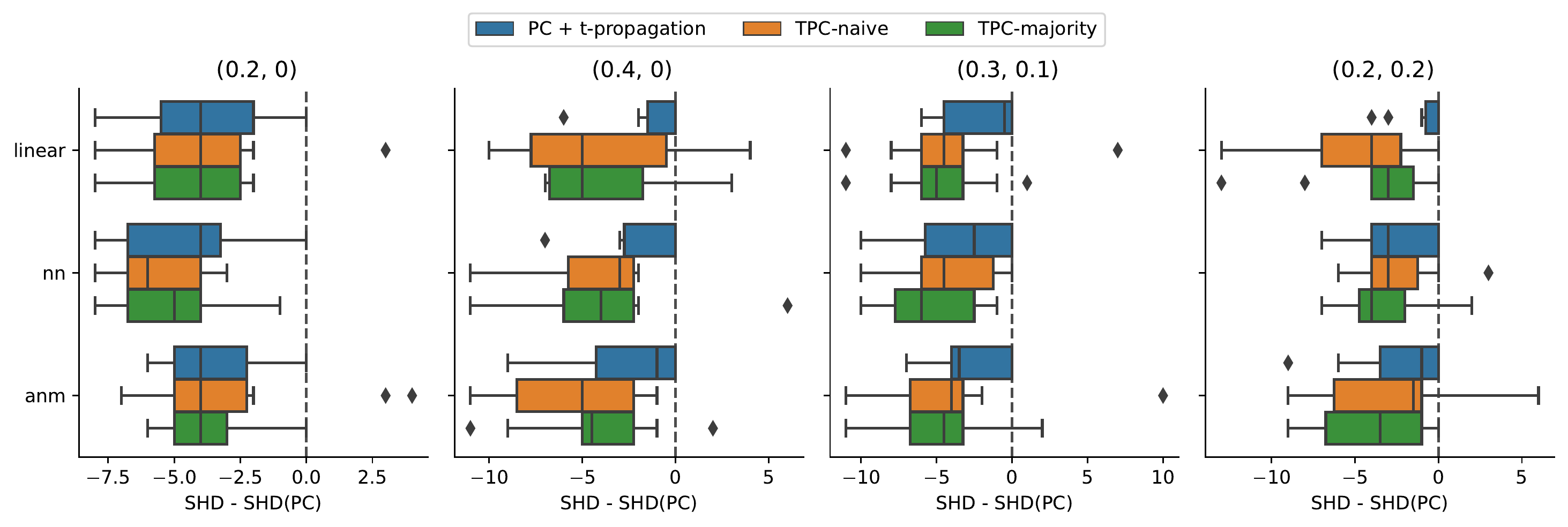}
    \caption{\vspace{-4mm}Results on simulated data with 3 types shown as improvement in SHD w.r.t. the PC baseline (lower is better). The title of each subplot indicates the ($p_{\text{inter}}, p_{\text{intra}}$) configuration. 
    }
    \label{fig:app-simulated-types3}
\end{figure}

\begin{figure}
    \centering
    \includegraphics[width=\textwidth]{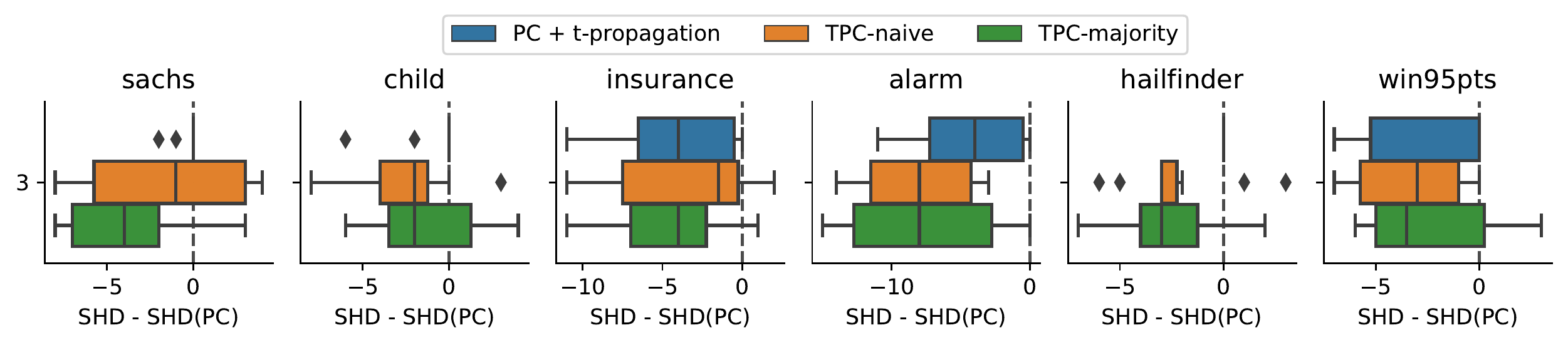}
    \caption{\vspace{-4mm}Results on pseudo-real data with 3 types shown as improvement in SHD w.r.t. the PC baseline (lower is better).  
    }
    \label{fig:app-pseudoreal-3types}
\end{figure}

\begin{figure}
    \centering
    \includegraphics[width=\textwidth]{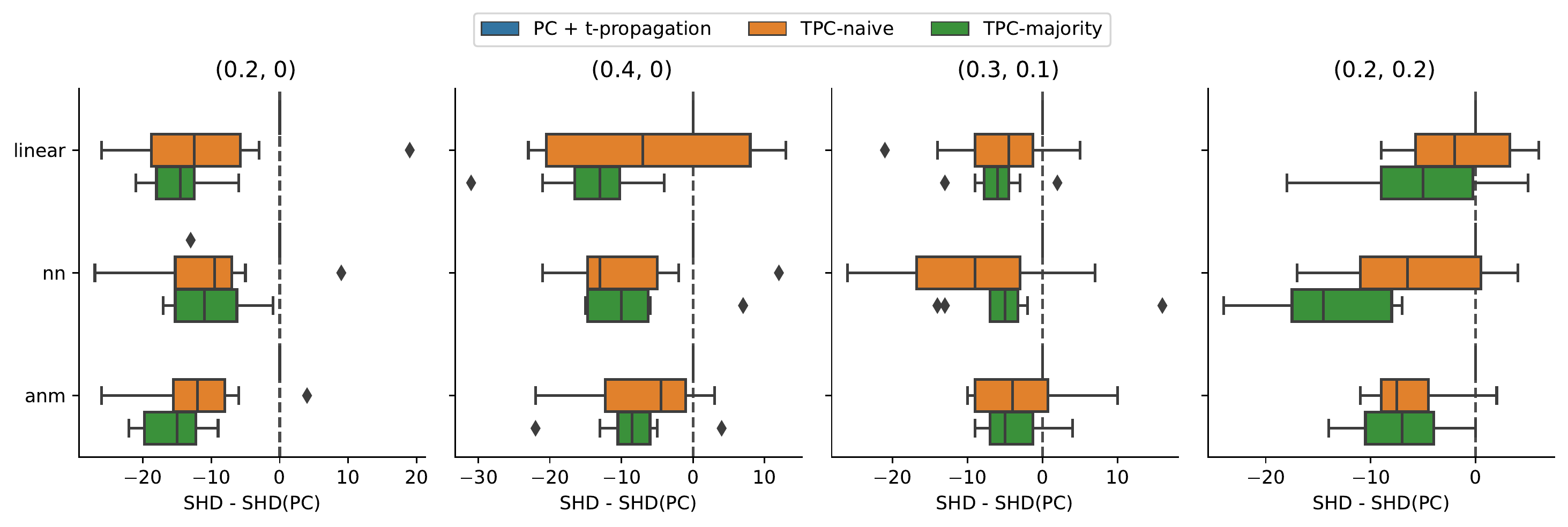}
    \caption{\vspace{-4mm}Results on simulated data with 50 vertices shown as improvement in SHD w.r.t. the PC baseline (lower is better). The title of each subplot indicates the ($p_{\text{inter}}, p_{\text{intra}}$) configuration. 
    }
    \label{fig:app-simulated-50nodes}
\end{figure}

\subsection{Hyperparameters}

The algorithms compared in this work all have the same hyperparameters. For each algorithm, we use the FIT~\citep{chalupka2018fast} as the conditional independence test with $\alpha = 0.01$. We use the default parameters for FIT, as defined in the \href{https://pypi.org/project/fcit/}{fcit} Python package (Version 1.2.0).

\section{Notes to practitioners} \label{app:note_practitioners}

\subsection{Typing variables without full knowledge of causation between types}
\label{app:examples_ordering_unknown}
It is reasonable to assume that experts who are able to attribute types to variables may have some \emph{a priori} intuition about how these types are related.
For example, \citet{shen2020challenges} claim that ``edges from biomarkers or diagnosis to demographic variables are prohibited''.
Our framework is compatible with this setting, in which the orientation of some t-edges may be known a priori. 
However, our results also hold in the case where such prior knowledge is either \emph{unavailable}, \emph{incomplete}, or \emph{unreliable}. 
We give two examples of practical settings where this may arise.

\paragraph{Example 1 - Time-related types with missing information.} Consider a setting where variables are collected over multiple days. This corresponds to a typical tiered background knowledge setting, where we let tiers correspond to types, i.e., the day on which variables were measured (e.g., Day 1 variables, …, Day n variables). Now, consider the case where the exact date of measurement for some types is missing (or unreliable), e.g., for privacy consideration. That is, we know which variables were collected simultaneously, but we do not know when. While it is not possible to order such types with respect to the others, we know that their inter-type causal relationships follow the arrow of time and that all variables in a type are at the same position in time. In contrast with standard tiered background knowledge, our framework allows for the inclusion of such partial information.

\paragraph{Example 2 - Entities as types.} In some use-cases, types could be used to represent distinct entities, each characterized by multiple measured variables. For example, such entities could be devices on a network for which telemetry data is available, users in a social network and their behavioral characteristics, or machines in a manufacturing production line and their input/output quantities (akin to \cite{marazopoulou2016causal}). In such settings, the nature of the entities is not necessarily sufficient for an expert to have an intuition about their ordering. For example, we may know that two sets of variables are from two distinct network devices, but not how these devices are related within the network topology. When it is reasonable to assume that entities interact in a directional manner (e.g., producer/consumer, influencer/follower relationships), our typing assumptions can be applied by grouping the variables of each entity into a distinct type.

\subsection{Variations and relaxations of the proposed typing assumption} \label{app:relaxations}

Here, we briefly touch on two variations of our proposed theoretical framework: 1) the case where the types of some variables are unknown and 2) the case where type consistency does not apply to a subset of the variables.

\paragraph{The types of some variables are unknown. } This case is adequately covered by our current theoretical framework. Variables with unknown types should be assigned to a unique type of which they are the only instance. This ensures that the type consistency assumption still constrains their interactions with other types. Alternatively, one could think of assigning all variables with unknown types to a single (common) type. However, this is incorrect, since it forces all such variables to interact with other types in the same direction.

\paragraph{Type consistency does not apply to some variables.}  This case requires a minor extension of our framework, which we have chosen to leave out to avoid complicating the presentation. One would need to add an ``exclusion set'' to the definition of t-DAGs (\cref{def:tdag}), which would contain vertices (i.e., variables) to which the type consistency assumption would not be applied in \cref{def:consistent tdag}. \cref{thm:tedges_zero} would still hold in this case, but the statement of \cref{cor:identification_pintra_zero} would need to be restricted to the case where the exclusion set is empty.

\end{document}